\documentclass[twoside]{article}

\usepackage[accepted]{aistats2022}
\usepackage{amsfonts}
\usepackage[T1]{fontenc}
\usepackage[utf8]{inputenc}
\usepackage{tabularx,ragged2e,booktabs,caption}
\newcolumntype{C}[1]{>{\Centering}m{#1}}

\usepackage[utf8]{inputenc}
\usepackage[T1]{fontenc}
\usepackage{url}
\usepackage{booktabs}
\usepackage{amsfonts, amsthm, amsmath}
\usepackage{mathtools}
\usepackage{nicefrac}
\usepackage{microtype}
\usepackage{graphicx}
\usepackage{subfigure}
\usepackage{booktabs} 
\usepackage{hyperref}
\usepackage{algorithmic}
\usepackage{algorithm}

\usepackage{hyperref}
\hypersetup{
    colorlinks=true,
    linkcolor=blue,
    citecolor=cyan,
}
\usepackage{xcolor}
\usepackage{amsthm,amssymb,amsmath}
\usepackage{mathrsfs}
\usepackage{mathtools}

\usepackage{enumitem}
\usepackage{natbib}
\usepackage{verbatim}
\usepackage{url}
\usepackage{thm-restate}
\usepackage{wrapfig}

\usepackage{bbm}

\usepackage{parskip}

\usepackage{graphbox}

\usepackage{subfigure}

\newtheoremstyle{definition}
{3pt} 
{3pt} 
{} 
{} 
{\bfseries} 
{.} 
{.5em} 
{} 

\theoremstyle{definition}

\begin{document}

%

%

\twocolumn[

\aistatstitle{Marginalized Operators for Off-policy Reinforcement Learning}

\aistatsauthor{Yunhao Tang \And Mark Rowland \And R\'emi Munos \And Michal Valko  }

\aistatsaddress{ DeepMind \And DeepMind \And DeepMind \And DeepMind } ]

\begin{abstract}
In this work, we propose marginalized operators, a new class of off-policy evaluation operators for reinforcement learning. Marginalized operators strictly generalize generic multi-step operators, such as Retrace, as special cases. Marginalized operators also suggest a form of sample-based estimates with potential variance reduction, compared to sample-based estimates of the original multi-step operators. We show that the estimates for marginalized operators can be computed in a scalable way, which also generalizes prior results on marginalized importance sampling as special cases. Finally, we empirically demonstrate that marginalized operators provide performance gains to off-policy evaluation and downstream policy optimization algorithms.
\end{abstract}

\section{Introduction}
\label{sec:intro}
In many applications of reinforcement learning (RL), it is useful to be able to learn about one policy using data generated by a different policy, such as exploratory data \citep{mnih2015human}, expert data \citep{hester2018deep} or even offline data  \citep{lange2012batch}; this is the problem of off-policy learning. To successully learn in such scenarios, off-policy algorithms must be able to safely deal with discrepancies between the data-generating policy and policy of interest. As a fundamental building block of generic off-policy algorithms, off-policy evaluation studies the problem of estimating  value functions of a target policy $\pi$ with data collected under behavior policy $\mu$. 

A distinction is often drawn in off-policy learning between \emph{online} and \emph{offline} learning. In the \emph{online} setting, where RL agents keep collecting new data, most prior work focuses on multi-step operator-based methods (e.g., \citep{precup2000eligibility,harutyunyan2016q,munos2016safe,rowland2020adaptive}). These methods equate policy evaluations to solving for fixed points of contractive operators. In this case, a central idea is \emph{bootstrapping}, where new estimates build on old estimates in an iterative fashion. As a result of contractive operators, the sequence of output from the algorithm forms increasingly accurate predictions to the true target values. This is especially desirable in many practical online setups where the target policy might slowly change over time (e.g., policy optimization), where predictions for the new policy could extract useful information from predictions for old policies.

On the other hand, in the \emph{offline} setting where no further data collection is possible, much work builds on importance sampling (IS) \citep{precup2000eligibility,thomas2015high,thomas2016data,liu2018breaking,nachum2019dualdice,uehara2019minimax,nachum2020reinforcement,xie2019towards,yang2020off}. Popular approaches for variance reduction in importance sampling are based on marginalized IS \citep{liu2018breaking,xie2019towards} which has also shown promises  even when combined with function approximations for high-dimensional input spaces \citep{nachum2019dualdice,nachum2020reinforcement,mousavi2020black}.
However, since the offline problems  only require a \emph{single} numerical prediction, most algorithms do not naturally incorporate the notion of \emph{bootstrapping} out-of-the-box. As a result, despite some recent efforts \citep{nachum2019algaedice}, it is in general challenging to directly apply such methods to \emph{online} off-policy learning.

Motivated by the disparity between these two lines of work, we propose marginalized operators, a new family of off-policy evaluation operators that generalize multi-step operators as special cases (Section~\ref{sec:marginalized}). Marginalized operators suggest new stochastic estimates to the equivalent multi-step operators, with connections to marginalized IS (Section~\ref{sec:understand}). Under this framework, we also consider \emph{estimated} marginalized operators (Section~\ref{section:estimate}), which can be computed with estimates in a scalable manner, and can be analyzed as estimators in their own right. Finally, we show that the new operators provide performance gains on both policy evaluation and downstream optimization (Section~\ref{sec:exp}).

Our discussions are limited to multi-step operators constructed as a weighted mixture of Bellman errors across different time steps. As a result, $Q^\pi$ is the unique fixed point of such operators; these exclude operators which explicitly bias the fixed point in exchange for faster contraction rate, such as the uncorrected $n$-step operator. See \citep{rowland2020adaptive} for a comprehensive discussion.
\section{Background}
\label{section:background}
\subsection{Markov decision processes}

Consider the setup of a Markov decision process (MDP) \citep{puterman2014markov} with an infinite horizon. At any discrete time $t\geq 0$, the agent is in state $x_t\in \mathcal{X}$, takes an action $a_t\in \mathcal{A}$. The agent first receives an immediate random reward $r_t=r(x_t,a_t)$ with mean $\bar{r}(x_t,a_t)$, and then transitions to a next state $x_{t+1}\sim p(\cdot|x_t,a_t)$. We assume rewards are deterministic, but most results extend naturally to the stochastic case. Below, we will discuss when such extensions do not hold. Let policy $\pi:\mathcal{X} \rightarrow \mathcal{P}(\mathcal{A})$ be a mapping from states to distributions over actions. Let $\gamma\in [0,1)$ be a discount factor, define the Q-function $Q^\pi(x,a) \coloneqq \mathbb{E}_\pi\left\lbrack\sum_{t=0}^\infty \gamma^t r_t \; \middle| \; x_0=x, a_0=a \right\rbrack$  and value function $V^\pi(x) \coloneqq \mathbb{E}_\pi\left\lbrack\sum_{t=0}^\infty \gamma^t r_t \; \middle| \; x_0=x \right\rbrack$. Here, $\mathbb{E}_\pi \left\lbrack \cdot\right\rbrack$ denotes that the trajectories $(x_t,a_t,r_t)_{t=0}^\infty$ are generated under policy $\pi$.

\subsection{Multi-step off-policy evaluation}
Consider off-policy evaluation where $\pi$ is the target policy and $\mu$ is the behavior policy, where we assume $\text{supp}\left(\pi(\cdot|x)\right)\subset \text{supp}\left(\mu(\cdot|x)\right),\forall x\in\mathcal{X}$. Given a trajectory $(x_t,a_t,r_t)_{t=0}^\infty$ generated under $\mu$ and a Q-function $Q$, we define the TD error at time $t$ as $\Delta_t^\pi Q \coloneqq \bar{r}_t + \gamma \mathbb{E}_{x'\sim p(\cdot|x_t,a_t)}\left\lbrack Q(x',\pi(x')) \right\rbrack - Q(x_t,a_t)$
. Here, we adopt the notation $Q(x,\pi(x))\coloneqq\mathbb{E}_{a\sim \pi(\cdot|x)}\left\lbrack Q(x,a) \right\rbrack$.  The multi-step  off-policy evaluation operators $\mathcal{R}^c$ \citep{munos2016safe} define the step-wise trace coefficient $c_t\in\mathbb{R}$ per time step $t$, where in general $c_t=c(\{x_s,a_s\}_{s\leq t})$ is a function of the  of the past $(x_s,a_s)_{s\leq t}$. The Q-function estimate $\mathcal{R}^c Q(x,a)$ at the starting pair $(x,a)$  is computed as
\begin{align}
    Q(x,a) + \mathbb{E}_{\mu}\left\lbrack \sum_{t\geq 0} \gamma^t (\Pi_{1\leq s\leq t} c_s) \Delta_t^\pi Q  \; \middle| \; x_0=x,a_0=a \right\rbrack,
    \label{eq:multi-step}
\end{align}
where we define $\left(\Pi_{1\leq s\leq t} c_s\right)=1$ when $t=0$. When $0\leq c_t \leq \frac{\pi(a_t|x_t)}{\mu(a_t|x_t)}$,
it can be shown that $Q^\pi$ is the unique fixed point to $\mathcal{R}^c Q=Q$ \citep{munos2016safe}. As an important example, let $c_t=\mathbb{I}[t\leq 0]$,  the operator $\mathcal{R}^c$ reduces to the one-step Bellman operator $\mathcal{T}^\pi Q(x,a) \coloneqq r_0  + \gamma \mathbb{E}_\pi\left\lbrack Q(x_1,\cdot)  \right\rbrack$. In this case, the traces $c_t$ are \emph{cut off} beyond the first time step, which prevents the algorithm from bootstrapping from the rest of the trajectory. In many cases, the coefficient $c_t=c(x_t,a_t)$ is Markovian if it only depends on $(x_t,a_t)$. Notable examples include importance sampling $c_t = \frac{\pi(a_t|x_t)}{\mu(a_t|x_t)}$, Retrace $c_t=\lambda \min\{\bar{c},\frac{\pi(a_t|x_t)}{\mu(a_t|x_t)}\}$ \citep{munos2016safe}, tree backup $c_t=\pi(a_t|x_t)$ \citep{precup2000eligibility} and $\text{Q}^\pi(\lambda)$ $c_t=\lambda$ \citep{harutyunyan2016q}.

\subsection{Off-policy evaluation via marginalized importance sampling}
We start by introducing the discounted visitation distribution  $d_{x,a}^\pi(x^\prime,a^\prime) \coloneqq (1-\gamma) \sum_{t\geq 0}\gamma^t \mathbb{P}_\pi(x_t=x^\prime,a_t=a^\prime|x_0=x,a_0=a)$ where $(x,a)$ are the starting state-action pair. The discounted visitation distribution $d_{x,a}^\pi(x^\prime,a^\prime)$ and value functions $Q^\pi(x,a)$ are related as follows \citep{puterman2014markov},
 \begin{align}
    Q^\pi(x,a) = (1-\gamma)^{-1} \mathbb{E}_{(x',a')\sim d_{x,a}^\pi}\left\lbrack r(x^\prime,a^\prime) \right\rbrack.
\end{align}
Assume the off-policy data is sampled under $d_{x,a}^\mu(x^\prime,a^\prime)$. Let $w_{x,a}^{\pi,\mu}(x^\prime,a^\prime)\coloneqq \frac{d_{x,a}^\pi(x^\prime,a^\prime)}{d_{x,a}^\mu(x^\prime,a^\prime)}$. One could express $Q^\pi(x,a)$ via marginalized IS \citep{xie2019towards,liu2018breaking},
\begin{align*}
    Q^\pi(x,a) = (1-\gamma)^{-1}\mathbb{E}_{(x^\prime,a^\prime)\sim d_{x,a}^\mu}\left\lbrack w(x^\prime,a^\prime) r(x^\prime,a^\prime) \right\rbrack.
\end{align*}  
For convenience, let $w^{\pi,\mu}\in\mathbb{R}^{\left(\mathcal{X}\times\mathcal{A}\right)\times \left(\mathcal{X}\times\mathcal{A}\right)}$ be a matrix 
such that $w_{x,a}^{\pi,\mu}(x',a')$ is the entry at $(x,a,x',a')$. Since marginalized IS ratios are generally unknown, it is necessary to construct estimates $w_\psi\approx w^{\pi,\mu}$. There are a number of algorithms which carry out the estimation in a scalable way, which we will detail in Section~\ref{section:estimate}. 

\paragraph{Remarks on notations.} Note that $Q^\pi: \mathbb{R}^{\mathcal{X}\times\mathcal{A}}\mapsto \mathbb{R}$ ($w^{\pi,\mu}: \mathbb{R}^{\left(\mathcal{X}\times\mathcal{A}\right)\times \left(\mathcal{X}\times\mathcal{A}\right)}\mapsto \mathbb{R}$) are by defintion functions. To facilitate derivations, we abuse notations and also treat them as vectors (matrices) such that $Q^\pi\in\mathbb{R}^{|\mathcal{X}||\mathcal{A}|}$($w^{\pi,\mu}\in\mathbb{R}^{|\mathcal{X}||\mathcal{A}|\times |\mathcal{X}||\mathcal{A}|}$). As such, $Q^\pi(x,a)$ can be both interpreted function evaluation and  vector indexing at $(x,a)$.

\section{Marginalized Off-Policy  Evaluation Operators}
\label{sec:marginalized}

The marginalized off-policy evaluation operator $\mathcal{M}^{w}: \mathbb{R}^{\mathcal{X}\times\mathcal{A}} \rightarrow \mathbb{R}^{\mathcal{X}\times\mathcal{A}}$ is defined such that its component at $(x,a)$ is evaluated as 
\begin{align}
    Q(x,a) + (1-\gamma)^{-1} \mathbb{E}_{(x',a')\sim d_{x,a}^\mu}\left\lbrack w_{x,a}(x^\prime,a^\prime)\Delta^\pi(x^\prime,a^\prime)  \right\rbrack,
    \label{eq:ba-retrace}
\end{align}
where $w_{x,a}(x',a')$ are called \emph{TD weights}. Define $\Delta^\pi(x,a)\coloneqq \bar{r}(x,a) + \gamma \mathbb{E}_{x^\prime\sim p(\cdot|x,a)}\left\lbrack Q(x',\pi(x')  \right\rbrack -Q(x,a)$ as $(x,a)$-dependent Bellman errors. Note the difference between $\mathbb{E}_\mu\left\lbrack \cdot\right\rbrack$ in Eqn~\eqref{eq:multi-step}, which is an expectation over trajectories $(x_t,a_t,r_t)_{t=0}^\infty$ under $\mu$; and $\mathbb{E}_{d_{x,a}^\mu}\left\lbrack \cdot\right\rbrack$ in Eqn~\eqref{eq:ba-retrace}, which is an expectation under the discounted distribution. 

Below, we will first characterize important properties of the marginalized operator. Then, we will show that the space of contractive marginalized operators contains the space of contractive multi-step operators.

\subsection{Properties of the marginalized operator} 

The following proposition summarizes a few important properties of the marginalized operators.
\begin{restatable}{proposition}{propmarginalizedop}\label{prop:marginalizedop}
For any TD weights $w$, the Q-function $Q^\pi$ is a solution to the fixed point equation $\mathcal{M}^wQ=Q$. For any $Q_1,Q_2\in \mathbb{R}^{\mathcal{X}\times\mathcal{A}}$, 
\begin{align*}
    \left|\mathcal{M}^{w}Q_1(x,a) - \mathcal{M}^{w}Q_2(x,a) \right| \leq \eta_{x,a}^w   \left\Vert Q_1-Q_2\right\Vert_\infty.
\end{align*}
Let $\delta_{x,a}\in \mathbb{R}^{\mathcal{X}\times\mathcal{A}}$ be the one-hot encoding of $(x,a)$ and let $d_{x,a}^w \in \mathbb{R}^{\mathcal{X}\times\mathcal{A}}$ such that $d_{x,a}^w(x',y')=w_{x,a}(x',y')d_{x,a}^\mu(x',a')$. Then define the \emph{residual error vector}
\begin{align*}
    E_{x,a}^w = (1-\gamma)\delta_{x,a} + \gamma (P^\pi)^T d_{x,a}^w - d_{x,a}^w,
\end{align*}
which characterizes how $d_{x,a}^w$ satisfies the \emph{balance equations}
\begin{align}
    (1-\gamma)\delta_{x,a} + \gamma (P^\pi)^T d - d = 0.
    \label{eq:bellman-discounted}
\end{align}
The local contraction rate is expressed as 
\begin{align}
    \eta_{x,a}^w = (1-\gamma)^{-1}\left\Vert E_{x,a}^w \right\rVert_1.
    \label{eq:local-contraction}
\end{align}
The above implies that the operator is contractive when $\max_{x,a}\left\lVert E_{x,a}^w\right\rVert_1<1-\gamma$.
\end{restatable}

Proposition~\ref{prop:marginalizedop} shows that
the local contraction rate $\eta_{x,y}^w$ is proportional to the $L^1$ norm of the residual error vector of $d_{x,y}^w$ when plugged into the balance equation. This means that in order for $\mathcal{M}^w$ to be contractive, we seek $w$ such that it approximately satisfies the balance equation and the residual error vector is small. 

Similar to the notation of $w^{\pi,\mu}$, we denote $w$ as the matrix of TD weights.  Though it is not straightforward to analytically characterize the set of $w$ such that $\mathcal{M}^w$ is contractive, we shed light on properties of such $w$ with some examples.

\paragraph{Marginalized IS ratios as a special case.}
The discounted visitation distribution $d_{x,a}^\pi$ is the only solution that satisfies the balance equation. When  $w_{x,a} = w_{x,a}^{\pi,\mu}$, since balance equations are satisfied exactly,  $\eta_{w_{x,a}^{\pi,\mu}}= 0$ and the contraction is instant $\mathcal{M}^{w^{\pi,\mu}} Q=Q^\pi,\forall Q$. Instead of requiring balance equations to be satisfied exactly, Proposition~\ref{prop:marginalizedop} suggests that there is a larger class of $w$ such that balance equations are approximately satisfied and $\mathcal{M}^w$ is contractive. Indeed, as we will see below, marginalized operators can recover all contractive multi-step operators as special cases.

\subsection{Multi-step off-policy evaluation operators as special cases}
\label{sec:specialcase}
We now elucidate the connections between marginalized operators with multi-step off-policy operators. The following result shows that when $w$ is chosen properly, the marginalized operators is equivalent to any given multi-step operator.
\begin{restatable}{proposition}{propequivalent}\label{prop:equivalent}
Given a multi-step operator $\mathcal{R}^c$ with step-wise trace coefficients $c_t$, define $w_{x,a}^c(x',a')$ as 
\begin{align}
    \frac{1-\gamma}{d_{x,a}^\mu(x^\prime,a^\prime)} \mathbb{E}_\mu\left\lbrack \sum_{t\geq 0} \gamma^t \left(\Pi_{1\leq s\leq t}c_s\right) \mathbb{I}[ x_t=x^\prime,a_t=a^\prime] \right\rbrack.
    \label{eq:marginalized-trace}
\end{align}
If $d_{x,a}^\mu(x',a')=0$ for some $(x',a')$, we can instead define $w_{x,a}^c(x',a')=0$. Let $w^c$ be the matrix form. When $w=w^c$, the two operators are equivalent, $\mathcal{M}^{w^c} = \mathcal{R}^c$.
\end{restatable}

 Proposition~\ref{prop:equivalent} implies that the space of all contractive marginalized operators contains all contractive multi-step operators. We formally summarize the result as follows.

\begin{restatable}{corollary}{corospace}\label{coro:space}
For any tuple $T=(p,r,\pi,\mu,\gamma)$,
Let $\mathcal{C}(T)$ be the space of all step-wise traces (Markovian or non-Markovian) such that $\mathcal{R}^c,c\in \mathcal{C}(T)$ is contractive; let $\mathcal{W}(T)$ be the space of all TD weights such that $\mathcal{M}^w,w\in \mathcal{W}(T)$ is contractive. Then
\begin{align*}
    \{\mathcal{R}^c , c\in \mathcal{C}(T)\}\subset\{\mathcal{M}^w , w\in \mathcal{W}(T)\}.
\end{align*}
\end{restatable}

As a concrete example of step-wise trace coefficient $c \in \mathcal{C}(T)$, consider the Markovian traces $c_t^{(\text{re})}\coloneqq \min(\frac{\pi(a_t|x_t)}{\mu(a_t|x_t)},1) \leq\frac{\pi(a_t|x_t)}{\mu(a_t|x_t)}$ that define the Retrace operators \citep{munos2016safe}. Let $w^{c^{(\text{re})}}$ be the equivalent marginalized trace. With some algebra, we can show its residual error vector is
\begin{align*}
    E_{x,a}^{w^{c^{(\text{re})}}} =
    \gamma \sum_{t=0}^\infty \gamma^t (\left(P^\pi\right)^T-\left(P^{\tilde{\pi}}\right)^T) (\left(P^{\tilde{\pi}}\right)^T)^t \delta_{x,a}\geq 0,
\end{align*}
where $\tilde{\pi}(a|x)\coloneqq c(x,a)\pi(a|x)$. We can interpret Retrace as imposing an additional yet implicit constraint on $c_t$, such that the residual error vector is non-negative $E_{x,a}^{w^{c}}\geq 0$. This is a stronger constraint than requiring the marginalized operator $\mathcal{M}^{w^c}$ to be contractive, which is equivalent to $\eta_{x,a}^{w^c} = (1-\gamma)^{-1}\left\Vert E_{x,a}^{w^{c}} \right\rVert_1<1$ as stated in Proposition~\ref{prop:contraction}. Indeed, as we will see next, by imposing weaker assumptions, marginalized operators contain a larger space of contractive operators than multi-step operators in general.

\subsection{Further characterizations of contractive marginalized operators}
The above discussion motivates the following question: does the space of contractive marginalized operators contains strictly more elements than contractive multi-step operators?  We have the following results.
\begin{restatable}{proposition}{propmore}\label{prop:more}
There exists tuples $T=(p,r,\pi,\mu,\gamma)$ such that either of the following holds
\begin{align*}
     & \text{(i)}\ \{\mathcal{R}^c , c\in \mathcal{C}(T)\}\subsetneq\{\mathcal{M}^w , w\in \mathcal{W}(T)\}, \\ 
    & \text{(ii)}\ \{\mathcal{R}^c , c\in \mathcal{C}(T)\}=\{\mathcal{M}^w , w\in \mathcal{W}(T)\}.
\end{align*}
\end{restatable}
Here, we provide some intuitions for case (i). One critical feature of multi-step operators is that the cumulative traces are multiplicative $C_t = \left(\Pi_{1\leq s\leq t} c_s\right)$. Assume a trajectory starting from $(x_0,a_0)$, if the cumulative trace $C_{t^\ast}=0$ at some time step $t^\ast$, then $C_t=0,\forall t\geq t^\ast$. However, by construction, marginalized operators might place TD weights $w_{x_0,a_0}(x_t,a_t)$ such that $w_{x_0,a_0}(x_{t^\ast},a_{t^\ast})=0$ and $w_{x_0,a_0}(x_{t'},a_{t'})\neq 0$ for some $t'>t^\ast$. In other words, marginalized operators could \emph{regenerate traces} while multi-step operators cannot. This implies that for such $w$, there does not exist $c\in\mathcal{C}(T)$ such that $\mathcal{R}^c=\mathcal{M}^w$. We provide specific instances where such phenomenon exist, see Appendix~\ref{appendix:proof} for the full derivations.

The above result bears important implications to Section~\ref{section:estimate}, where we apply operators $\mathcal{M}^{w_\psi}$ with parameterized TD weights $w_\psi$. They could be interpreted as directly parameterizing the space of contractive marginalized operators, without necessarily having any multi-step equivalents.

\section{Understanding Marginalized Off-Policy Evaluation Operators}
\label{sec:understand}

We have seen that by properly selecting $w$, marginalized operators recover multi-step operators as special cases. We provide insights on  marginalized operators from a few different perspectives. We start with some background.

\subsection{Stochastic estimates of evaluation operators}
\label{sec:vs}

Since operators are defined in expectations, a naive way to construct stochastic estimates is to directly draw samples from the expectations and compute empirical averages. For example, given a trajectory $(x_t,a_t)_{t=0}^\infty$ starting from $x_t=x,a_t=a$, a stochastic estimate to $\mathcal{R}^cQ(x,a)$ is
\begin{align*}
    \hat{\mathcal{R}}^c Q(x,a) = Q(x,a) + \sum_{t=0}^{\infty} \gamma^t \left(\Pi_{1\leq s \leq t} c_s\right) \hat{\Delta}_t, 
\end{align*}
where $\hat{\Delta}_t = r_t + \gamma Q\left( x_{t+1},\pi(x_{t+1}) \right) - Q(x_t,a_t)$ are estimates of Bellman errors.
We call this \emph{trajectory based} estimate as the estimate sums over data over the entire trajectory. We could also define a \emph{random time based} estimate with a random time $\tau$ such that $P(\tau=n)=(1-\gamma)\gamma^n$ for $n\geq 0$. 
\begin{align*}
    \hat{\mathcal{R}}_\tau^c Q(x,a) = Q(x,a) + (1-\gamma)^{-1} \left(\Pi_{1\leq s \leq \tau} c_s\right) \hat{\Delta}_\tau. 
\end{align*}
Both estimates are unbiased. Similarly, we define unbiased stochastic estimates for the marginalized evaluation operators, such that their expectations are $\mathcal{M}^wQ(x,a)$. 
\begin{align*}
    \hat{\mathcal{M}}^w Q(x,a) &= Q(x,a) + \sum_{t=0}^{\infty} \gamma^t w_{x,a}(x_t,a_t) \hat{\Delta}(x_t,a_t), \\
     \hat{\mathcal{M}}_\tau^w Q(x,a) &= Q(x,a) + (1-\gamma)^{-1} w_{x,a}(x_\tau,a_\tau) \hat{\Delta}(x_\tau,a_\tau).
\end{align*}

\subsection{Connections to conditional importance sampling}
\label{sec:cond-is}

Interestingly, the conversion of the step-wise trace coefficient $c_t$ into equivalent TD weights $w_(x,a)^c$ as defined in Eqn~\eqref{eq:marginalized-trace} is closely related to condition importance sampling (IS) \citep{liu2019understanding,rowland2020conditional}.

\begin{restatable}{proposition}{propcond}\label{prop:conditional-is}
Let $\tau$ be an integer-valued random time, such that $P(\tau=n)=(1-\gamma)\gamma^n,\forall n\geq 0$. For any step-wise trace coefficient $c_t$, its equivalent TD weights $w(x^\prime,a^\prime)$ is
\begin{align*}
    w_{x,a}^c(x',a') = \mathbb{E}_{\mu,\tau}\left\lbrack\left( \Pi_{1\leq s\leq \tau} c_s\right) \; \middle| \; x_\tau=x^\prime,a_\tau=a^\prime  \right\rbrack.
\end{align*}
\end{restatable}
In other words, $w_{x,a}^c(x^\prime,a^\prime)$ is the conditional expectation of the random cumulative traces $\left(\Pi_{1\leq s\leq \tau} c_s\right)$ conditional on the event $x_\tau=x^\prime,a_\tau=a^\prime$. In general, conditional IS is a useful technique for variance reduction \citep{casella2002statistical}, because for any two random variables $x,a$, $\mathbb{V}\left\lbrack X\right\rbrack\geq \mathbb{V}\left\lbrack \mathbb{E}\left\lbrack X|Y\right\rbrack \right\rbrack$. This implies a variance reduction property of stochastic estimates to the marginalized operators.
\begin{restatable}{corollary}{corocondoperator}\label{coro:conditional-is-operator} Assume that both state transitions and rewards are deterministic.  While having the same expectations, the random-time based estimate for the marginalized operator has smaller variance compared to that of the multi-step operator,
\begin{align*}
    \mathbb{V}\left\lbrack\hat{\mathcal{M}}_{\tau}^{w^c}Q(x,a)\right\rbrack \leq \mathbb{V}\left\lbrack\hat{\mathcal{R}}_\tau^c Q(x,a)\right\rbrack.
\end{align*}
\end{restatable}

Importantly, Corollary~\ref{coro:conditional-is-operator} assumes that both state transitions and rewards are deterministic; there is no provable variance reduction when, e.g., the rewards are stochastic.
In Appendix~\ref{appendix:operators}, we graphically present the relations between the four estimates to different operators introduced above. 

\paragraph{Remarks on trajectory based estimates.}
Trajectory based estimates usually have smaller variance than the random time based counterparts. This is because
\begin{align*}
    \hat{M}^{w^c}Q(x_0,a_0) &= \mathbb{E}\left\lbrack  \hat{\mathcal{M}}_{\tau}^{w^c}Q(x_0,a_0) \; \middle| \; (x_t,a_t,r_t)_{t=0}^\infty  \right\rbrack, \nonumber \\ 
    \hat{\mathcal{R}}^c Q(x_0,a_0) &= \mathbb{E}\left\lbrack  \hat{\mathcal{R}}_\tau^c Q(x_0,a_0) \; \middle| \; (x_t,a_t,r_t)_{t=0}^\infty  \right\rbrack.
\end{align*}
Though Collorary~\ref{coro:conditional-is-operator} shows the order of variance between random time based estimates, the order of variance of the trajectory based estimates $\hat{\mathcal{R}}^cQ(x,a)$ vs. $\hat{\mathcal{M}}^{w^c} Q(x,a)$ are not clear. Similar results have been observed in
 \citep{liu2019understanding}, where they show that marginalized IS via extended conditional expectations \citep{bratley2011guide} does not necessarily reduce variance. Nevertheless, in practice, estimates to marginalized operators usually reduce variance as evidenced empirically \citep{liu2018breaking}.

 \paragraph{Trade-off of practical estimates.}
In practice, TD weights $w^c$ are unknown and need to be estimated $w_\psi\approx w^c$. As a concrete example, consider $c_t=\frac{\pi(a_t|x_t)}{\mu(a_t|x_t)}$ and $w^c=w_{x,a}^{\pi,\mu}$. To clarify the trade-off, let $Q\equiv 0$. In this case, $\hat{\mathcal{R}}^c Q(x,a)=\sum_{t\geq 0} \gamma^t (\Pi_{1\leq s\leq t} \frac{\pi(a_s|x_s)}{\mu(a_s|x_s)}) r_t$ \citep{precup2000eligibility}, which might suffer from high variance due to the product of IS ratios \citep{liu2019understanding}. On the other hand, $\hat{\mathcal{M}}^{w_\psi}Q(x,a)=(1-\gamma)^{-1}\sum_{t=0}^\infty w_\psi(x_t,a_t) r_t \approx \hat{\mathcal{M}}^{w^c}Q(x,a) $ where $w_\psi\approx w^c$ is a parametric estimate \citep{liu2018breaking}. As argued in prior work, the latter has lower variance due to marginalized IS but at the cost of the bias in the estimate $w_\psi$. Overall, moving from the multi-step operator $\hat{\mathcal{R}}^c$ to its estimated marginalized counterpart $\hat{\mathcal{M}}^{w_\psi}$, one trade-offs variance with potential bias due to imperfect estimates of $w^c$ \citep{rowland2020adaptive}. For general step-wise traces $c_t$ and $w^c$, this trade-off should still hold.  As such, the quality of $w_\psi\approx w^c$ determines the quality of downstream updates. We will discuss in
Section~\ref{section:estimate} how to characterize such effects and estimate $w_\psi$.
 
\paragraph{Related work on conditional IS.} \citep{rowland2020conditional} interprets a large class of off-policy evaluation algorithms as a two-stage process: \textbf{(1)} start with an initial estimate; \textbf{(2)} compute the conditional IS of the estimate w.r.t. some conditioning variables. State-action pairs $(x,a)$ are popular choices of the conditioning variables, e.g., when applied to marginalized IS \citep{xie2019towards,liu2018breaking} and eligibility traces \citep{van2020expected}. In this work, we interpret marginalized operators as applying a similar procedure to step-wise traces $c_t$ to derive TD weights $w^c$.

\paragraph{Extension to V-trace operators.} So far, our discussion has focused on off-policy evaluation for Q-functions. By interpreting the TD weights as conditional IS of step-wise traces, we can extend this approach to off-policy evaluation of value functions such as V-trace operators \citep{espeholt2018impala}. see Appendix~\ref{appendix:vtrace} for detailed results.

\subsection{Policy evaluation via linear programs and its connections to contractions }
\label{sec:lp}
The linear programming (LP) formulation of MDPs \citep{de2003linear,puterman2014markov} is an important framework for policy evaluation, which gives rise to a large number of recent work on marginalized off-policy evaluation (e.g., see \citep{nachum2020reinforcement}). Here, we explore how the notion of contraction is in fact consistent with the LPs. We will see that this offers a new way to interpret LP formulation for policy evaluation, and might pave the way for new algorithms.

\paragraph{Dual LP for policy evaluation.} Consider the evaluation of $Q^\pi(x,a)$. Denote $R\in \mathbb{R}^{\mathcal{X}\times\mathcal{A}}$ as the reward vector $R(x,a)=r(x,a)$. We directly start with the dual LP where $d\in \mathbb{R}^{\mathcal{X}\times\mathcal{A}}$ are dual variables. The dual LP for policy evaluation is \citep{puterman2014markov}
\begin{align}    \left\{                \begin{array}{ll}            \min \  (1-\gamma)^{-1}d^T  R\\            (1-\gamma)\delta_{x,a} + \gamma (P^\pi)^T d - d = 0\\            \end{array}              \right.              \label{eq:eval-duallp}\end{align} 
Since the equality constraints are essentially the balance equations defined in Eqn~\eqref{eq:bellman-discounted}, the single feasible (optimal) solution is $d^\ast=d_{x,a}^\pi$.

\paragraph{Sequence of relaxed LPs as repeated application of contractive operators.}

We start by assuming an iterative algorithm, where at iteration $t$ we have access to Q-function estimate $Q_t \in\mathbb{R}^{\mathcal{X}\times\mathcal{A}}$. At iteration $t+1$, consider the dual LP (Eqn~\eqref{eq:eval-duallp}) for each $(x,a)$. We augment its objective function as follows \begin{align}    \left\{                \begin{array}{ll}            \min \ Q_t^T \delta_{x,a} + (1-\gamma)^{-1}d^T  (R + \gamma (P^\pi)^T Q_t - Q_t) \\            (1-\gamma)\delta_{x,a} + \gamma (P^\pi)^T d - d = 0\\            \end{array}              \right.              \label{eq:eval-duallp-augmented}\end{align}
 Note that the augmented dual LP (Eqn~\eqref{eq:eval-duallp-augmented}) has the same optimal solution as the original dual LP (Eqn~\eqref{eq:eval-duallp}) because both of their feasible region contains only $d_{x,a}^\pi$. Let $\eta\in[0,1)$ be a scalar constant. We relax the constraints of the above dual LP as follows,
\begin{align}    \left\{                \begin{array}{ll}            \min \ Q_t^T \delta_{x,a} + (1-\gamma)^{-1}d^T (R + \gamma P^\pi Q_t - Q_t)\\            (1-\gamma)\delta_{x,a} + \gamma (P^\pi)^T d - d \leq (1-\gamma)u\\             (1-\gamma)\delta_{x,a} + \gamma (P^\pi)^T d - d \geq -(1-\gamma)u\\            1^T u \leq \eta, u\geq 0\\            \end{array}              \right.              \label{eq:eval-duallp-augmented-relaxed}\end{align}
We name the above relaxed problem $\text{LP}^{(t)}(x,a)$. The feasible region of the relaxed dual LP (Eqn~\eqref{eq:eval-duallp-augmented-relaxed}) is expanded into a non-trivial polyhedron $\mathcal{D}_{x,a}$ when $\eta>0$. Instead of requiring balance equations to hold exactly, violations are allowed and their magnitude is controlled by $\eta$. Define $Q_{t+1}(x,a)$ to be the objective value of Eqn~\eqref{eq:eval-duallp-augmented-relaxed}. The following result relates the sequence of LP objectives to contraction.
\begin{restatable}{proposition}{proplp}\label{prop:trace-bellman}  The following holds for the sequence of values produced by relaxed LPs,
\begin{align*}    \left\lVert Q_{t+1}-Q^\pi\right\Vert_\infty \leq \eta  \left\lVert Q_t-Q^\pi\right\Vert_\infty.\end{align*}\end{restatable}

To better understand the above result, note that the feasible region $\mathcal{D}_{x,a}$ effectively characterizes all TD weights $w$ that $\mathcal{M}^w$ is contractive with rate at most $\eta$. In particular, 
\begin{align*}
    \mathcal{D}_{x,a} = \left\{w_{x,a}\odot d_{x,a}^\mu | \eta_{x,a}^w \leq \eta \right\}
\end{align*}
where $\odot$ is the element-wise product of vectors. As we show below, the iterative process $Q_t \rightarrow Q_{t+1}$ is equivalent to applying contractive operators for policy evaluation
\begin{restatable}{corollary}{coroequivalence}\label{coro:equivalence} For any $(x,a)$, let $w_{x,a}^\ast=\frac{d^\ast}{d_{x,a}^\mu}\in\mathbb{R}^{\mathcal{X}\times\mathcal{A}}$ and $d^\ast$ is the optimal solution to $\text{LP}^{(t)}(x,a)$, then $\eta_{x,a}^{w_{x,a}^\ast}\leq \eta$ and
\begin{align*}
   Q_{t+1}(x,a) = \mathcal{M}^{w_{x,a}^\ast} Q_t(x,a).
\end{align*}
\end{restatable}
In other words, instead of directly outputting $Q^\pi(x,a)$ by solving $\text{LP}^{(0)}(x,a)$, this iterative algorithm solves relaxed problems and generates a sequence of LP values $Q_t\rightarrow Q^\pi$ by implicitly applying operator $\mathcal{M}^{w_{x,a}^\ast}$.  

\paragraph{Related ideas.} The idea to reduce solving a single LP into a solving a sequence of relaxed LPs has been explored (e.g., in \citep{peters2010relative,bas2020logistic}). They consider the LP for policy optimization, and relax constraints by projecting them onto low-dimensional spaces. This is orthogonal to the box relaxation in Eqn~\eqref{eq:eval-duallp-augmented-relaxed}.

\section{Estimating TD Weights}
\label{section:estimate}

As previously discussed marginalized operators can achieve variance reduction compared to the equivalent multi-step operators (Corollary~\ref{coro:conditional-is-operator}). This poses a practical question: given a multi-step operator $\mathcal{R}c^{\pi,\mu}$, how to find its marginalized equivalent $\mathcal{M}^{w^c}$? In other words, how to estimate $w^c$ from $c_t$? Given a specific step-wise trace coefficient $c_t$, we seek an algorithm that estimates the equivalent TD weights $w_\psi\approx w_{x,a}^c,\forall (x,a)$. Throughout the discussion, we focus on Markovian step-wise traces that define Retrace operators $0\leq c(x_t,a_t)\leq\frac{\pi(a_t|x_t)}{\mu(a_t|x_t)}$ \citep{munos2016safe}. 

We adapt the TD-learning based method introduced in \citep{liu2018breaking} and derive algorithms to estimating TD weights for generic Markovian step-wise traces. We define $\tilde{\pi}(a|x)\coloneqq \mu(a|x)c(x,a)$
and a scoring function (also called a critic or discriminator) $\mathbf{q}\in\mathbb{R}^{\mathcal{X}\times\mathcal{A}}$. Consider the loss function,
\begin{align}
    L(\mathbf{q},w_\psi)&\coloneqq 
    (1-\gamma)q(x,a)+
    \mathbb{E}_{(x',a')\sim d_{x,a}^\mu}\left\lbrack \Gamma(x',a') \right\rbrack.
        \label{eq:td-estimate}
\end{align}
Here, we define $\Gamma(x',a')$ as
\begin{align*}
 w(x',a')\left(\gamma \mathbb{E}_{x''\sim p(\cdot|x',a')} \left\lbrack q\left(x'',\tilde{\pi}(x'')\right)\right\rbrack  - q(x',a')\right).
\end{align*}
We now show a few important properties of the loss function that motivate its use as a practical objective,
\begin{restatable}{lemma}{proplemma}\label{lemma:loss}
Given any two class of scoring functions $\mathcal{Q}_1\subset \mathcal{Q}_2$, 
$
    \max_{\mathbf{q}\in\mathcal{Q}_1} L(\mathbf{q},w) \leq  \max_{\mathbf{q}\in\mathcal{Q}_2} L(\mathbf{q},w) ,\forall w 
$
In addition, the TD weights achieve the global optimal
$w_{x,a}^{c} =\arg\min_w \max_{\mathbf{q}\in\mathcal{Q}} L(\mathbf{q},w)$ for any $\mathcal{Q}$.
\end{restatable}
The above result motivates the use of the saddle point optimization objective to search for $w_\psi\approx w_{x,a}^c$. Consider optimizing the following objective jointly with respect to $\psi$ and $\mathbf{q}$,
\begin{align}
    \min_\psi \max_{\mathbf{q}\in\mathcal{Q}} L(\mathbf{q},w_\psi).
    \label{eq:minmax}
\end{align}

The outcome of the optimization $\psi$ can then be used as an approximation $w^\psi\approx w_{x,a}^c$. To characterize the quality of the approximation, note that when $\mathcal{Q}$ contains a large set of scoring functions, the solution $\psi^\ast$ to Eqn~\eqref{eq:minmax} should be closer to $w_{x,a}^c$. This is captured by the following result.
\begin{restatable}{proposition}{propestimate}\label{prop:estimate}
For any sub-probability measure $\tilde{\pi}$, Let $T_{\tilde{\pi}}(x^\prime,a^\prime|x,a)\coloneqq  p(x^\prime|x,a) \tilde{\pi}(a^\prime|x^\prime)$ be the one-step marginal transition probability. Let $T_{\tilde{\pi}}^t(x^\prime,a^\prime|x,a)$ be the $t$-time composition of $T_{\tilde{\pi}}(\cdot|x,a)$. Given a target state-action pair $(x^\ast,a^\ast)$, define the scoring function $q(x,a,x^\ast,a^\ast)\coloneqq \sum_{t\geq 0}\gamma^t T_{\tilde{\pi}}^t(x,a|x^\ast,a^\ast)$. Then if $\mathcal{Q}_T(x,a,x^\ast,a^\ast)=\{\pm q(x,a,x^\ast,a^\ast)\}\subset \mathcal{Q}$, the following holds,
\begin{align*}
  |w_\psi(x^\ast,a^\ast)-w_{x,a}^c(x^\ast,a^\ast)| \leq \frac{ \max_{\mathbf{q}\in\mathcal{Q}} L(\mathbf{q},w_\psi)}{d_{x,a}^\mu(x^\ast,a^\ast)}.
\end{align*}
\end{restatable}

When $c_t=\frac{\pi(a_t|x_t)}{\mu(a_t|x_t)}$,  Proposition~\ref{prop:estimate} reduces to Theorem 6 in \citep{liu2018breaking} as a special case.
In practice, however, it might not be necessary to estimate accurately at each point $(x,a)$. This is because for practical purposes, we only need the downstream operator $\mathcal{M}^{w_\psi}$ to be contractive. The following section discusses how the objective can be directly used for optimizing the contraction rate.

\subsection{Optimizing for the contraction rate}

So far, we have discussed how to optimize the parameter $\psi$ such that $w_\psi$ matches a particular $w^c$ for some particular trace coefficient $c_t$. The following result shows that how one could directly minimize the local contraction rate $\eta_{x,a}^{w_\psi}$, without the need to commit to any trace coefficient $c_t$.

\begin{restatable}{proposition}{propcontraction}\label{prop:contraction}
Assume that $\mathcal{Q}_b=\{\pm \delta(x=x^\ast,a=a^\ast),\forall (x^\ast,a^\ast)\}\subset\mathcal{Q}$. When $c_t=\frac{\pi(a_t|x_t)}{\mu(a_t|x_t)}$ and $w^c=w^{\pi,\mu}$, the contraction rate of $\mathcal{M}^{w_\psi}$ is upper bounded as
$
  \eta_{x,a}^{w_\psi} \leq  \max_{\mathbf{q}\in \mathcal{Q}} L(\mathbf{q},w_\psi).
$
\end{restatable}

Even when the TD weights are not estimated perfectly, the estimated marginalized operator $\mathcal{M}^{w_\psi}$ are still properly defined contractive operators. The above result also implies that in the presence of estimation errors, $\mathcal{M}^{w_\psi}$ could still be contractive
even when the TD weights $w$ do not exactly match any weights $w^c$ for any particular trace coefficient $c_t$.
As a result, repeated application of the operator still converges to the correct value. This differs from how prior work interprets imperfect weight estimates (e.g., see \citep{liu2018breaking}) as incurring errors to the final prediction in the offline case.

\paragraph{Remarks on $\mathcal{Q}_b$.} Compared to $\mathcal{Q}_T(x,a,x^\ast,a^\ast)$,  $\mathcal{Q}_b$ is much more straightforward to parameterize in practice. For example, consider a neural network $f_\eta$ which takes $(x,a)$ as input and takes $\text{tanh}$ as the output activation: $\text{tanh}(f_\eta(x,a))\in [-1,1]$. When $f_\eta$ is expressive enough, it parameterizes the convex hull of $\mathcal{Q}_b$.

\paragraph{Other methods for marginalized estimations.} Recently, there is a growing interest in marginalized estimation for off-policy evaluation. Besides TD-learning methods, other notable examples include Fenchel-duality based methods \citep{nachum2019dualdice,nachum2019algaedice,nachum2020reinforcement} and kernel machines \citep{mousavi2020black}. In Appendix~\ref{appendix:fenchel-dual}, we derive a Fenchel-duality based approach to estimating TD weights, which naturally extends the original work \citep{nachum2019dualdice}.

\section{Experiments}
\label{sec:exp}

We start with a few tabular examples to build better understanding of the empirical properties of marginalized operators. For all tabular MDPs, we adopt the tabular representation when learning TD weights. Then we evaluate the potential benefits of marginalized operators when combined with multi-step deep RL algorithms. In this latter case, the TD weights are estimated with function approximations.

 \begin{figure*}[h]
    \centering
    \subfigure[Number of actions]{\includegraphics[keepaspectratio,width=.18\textwidth]{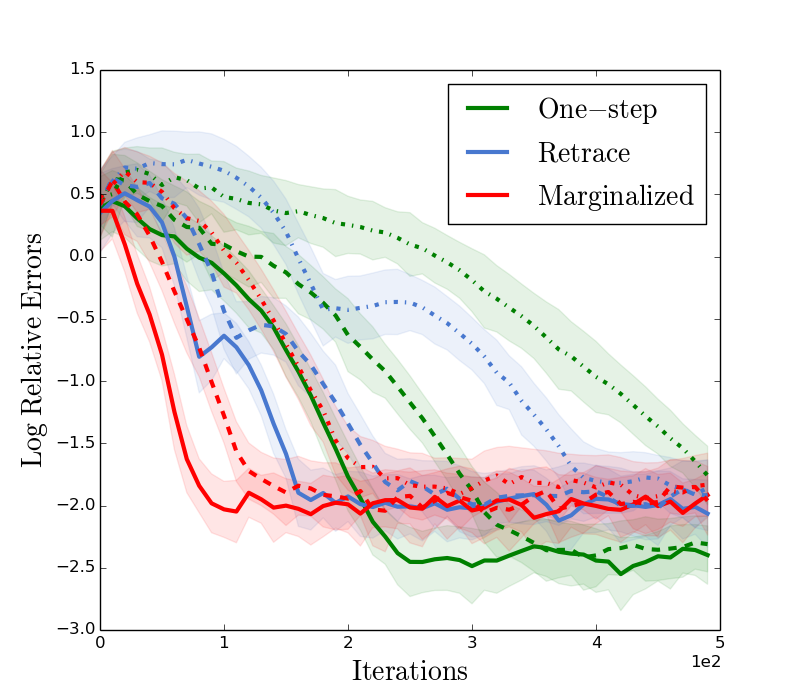}}
    \subfigure[Horizon]{\includegraphics[keepaspectratio,width=.18\textwidth]{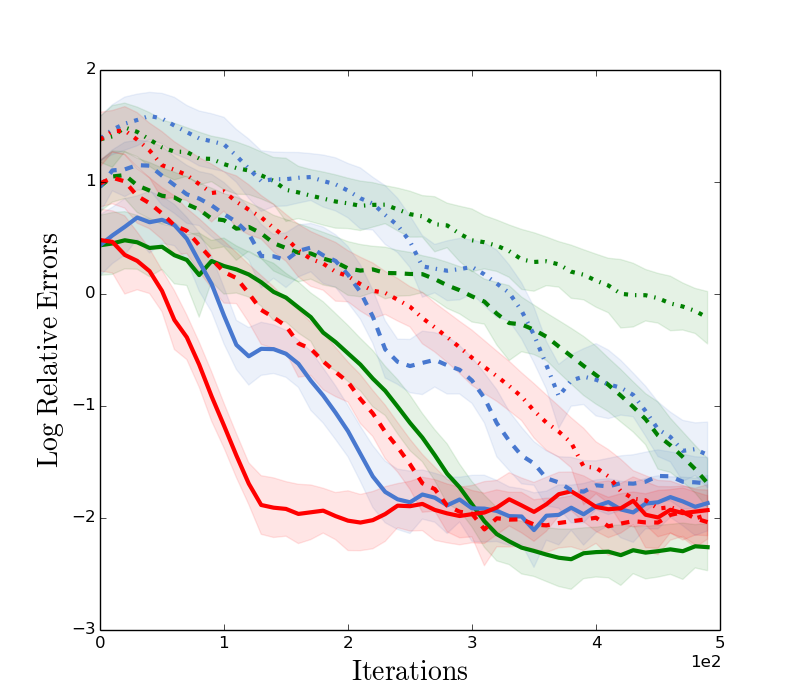}}
    \subfigure[Off-policy level]{\includegraphics[keepaspectratio,width=.18\textwidth]{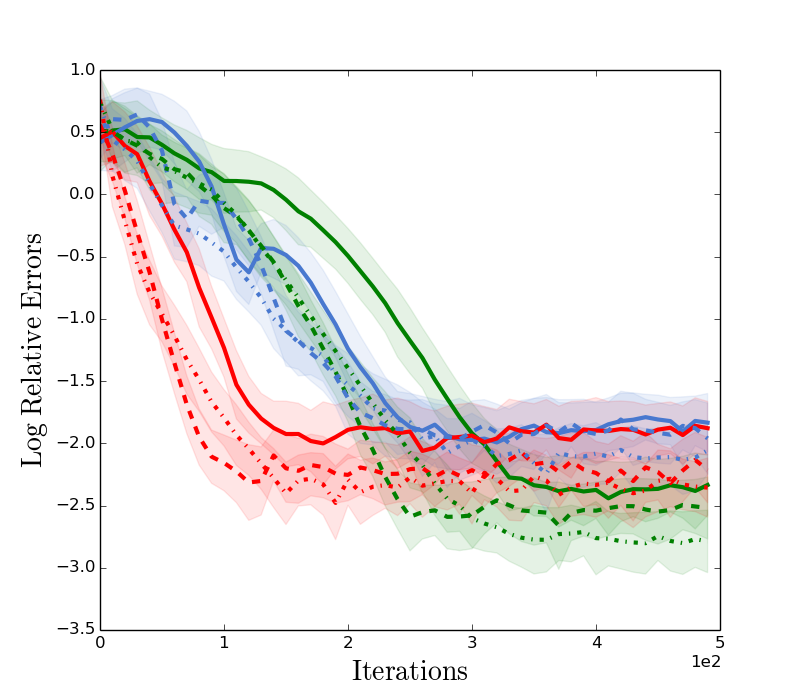}}
    \subfigure[Noise level]{\includegraphics[keepaspectratio,width=.18\textwidth]{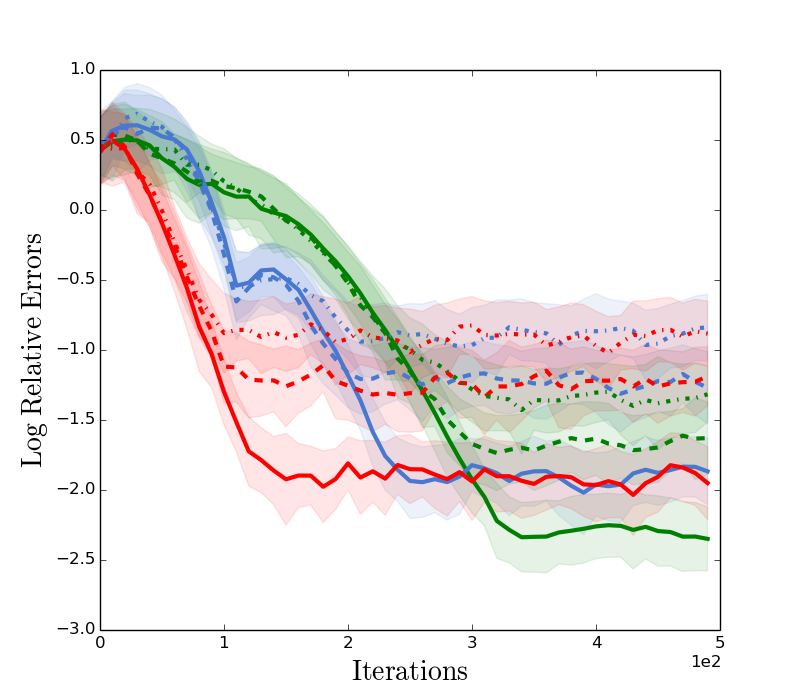}}
    \subfigure[Truncation level]{\includegraphics[keepaspectratio,width=.18\textwidth]{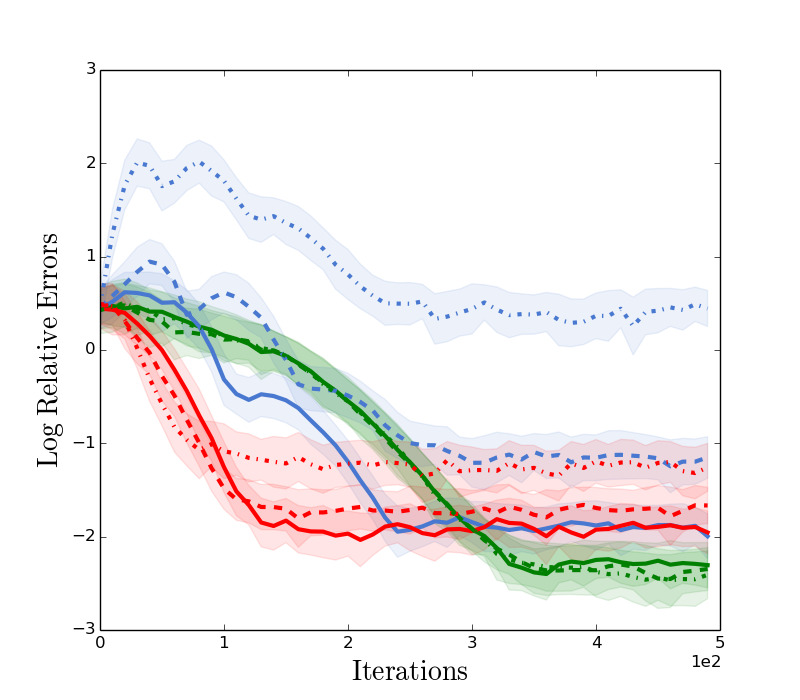}}
    \caption{Comparison of baseline operators on chain MDPs. Each curve is averaged over $100$ random seeds. The y-axis shows the relative estimation errors in log scale. The x-axis shows the number of iterations. In each plot, we vary one hyper-parameter of the MDP shown by curves with different line styles. The line styles and their corresponding hyper-parameters are shown in Table~\ref{table:chain}.}
    \label{fig:contraction}
\end{figure*}

\subsection{Chain MDP}
 
  Consider a chain MDP. The reward is zero unless at the rightmost state. At the rightmost state, the reward for action $a \in \mathcal{A}$ is $\mathcal{N}(\mu_a,\sigma^2)$ where $\mu_a=0$ for all but one action $a^\ast$ where $\mu_{a^\ast}=1$. The episode starts with the leftmost state. For all states, the transition goes to the state to its right with probability $1$, no matter what action is taken, until at the rightmost state when the episode terminates. Due to the dynamics of the problem, the episodic horizon is $T\equiv |\mathcal{X}|$. We consider the target policy $\pi$ as a deterministic policy of choosing action $a=a^\ast$ at all time. We start with a uniformly random policy $u$ and construct the behavior policy as $\mu=\beta \pi + (1-\beta)u$ where $\beta\in[0,1]$ controls the off-policy level. The problem is on-policy by setting $\beta=1$. For further details, see Appendix~\ref{appendix:experiment}.
 
 \begin{table}[]
    \vskip 0.1in
    \begin{center}
    \footnotesize
    \begin{sc}
    \begin{tabular}{c|c|c|c}\toprule[1.5pt]
        \bf Line styles & \bf  Solid & \bf Dashed & \bf Dashed-dot  \\\midrule
\# Actions $|\mathcal{A}|$ & $5$ & $10$ & $20$  \\
Horizon $T$ & $10$ & $20$ & $30$   \\
Off-policy $\beta$ & $0$ & $0.3$  & $0.7$   \\
Noise $\sigma$ & $0.1$ & $0.5$  & $1.0$   \\
Truncation $\bar{c}$ & $1$ & $2$  & $5$   \\
 \bottomrule[1.46pt]
    \end{tabular} \par
    \end{sc}
    \end{center}
    \caption{Parameter tables of the chain MDP. This table shows the line styles and their corresponding parameters in Figure~\ref{fig:contraction}.}
    \label{table:chain}
\end{table}
 
 To investigate the impact of different hyper-parameters on the experiment results, we vary the number of actions $|\mathcal{A}|$, the horizon $T$, the off-policy level $\beta$, the noise level $\sigma$ as they capture different aspects of the MDP. In each sub-plot we vary only one parameter and keep others at the default values. Curves with different line styles correspond to different values of a given parameter, shown in Table~\ref{table:chain}. The default hyper-parameters of the experiments are in the leftmost column of the table. We compare three baselines: \textbf{(1)} one-step operator  $\mathcal{T}^\pi Q(x,a)=r(x,a)+\gamma\mathbb{E}[Q(x',\pi(x'))]$, which we recall can be obtained as a special case of Retrace when $\bar{c}=0$; \textbf{(2)} Retrace ($c_t=\min\{\bar{c},\frac{\pi(a_t|x_t)}{\mu(a_t|x_t)}\}$ where $\bar{c}=1$ by default) and \textbf{(3)} marginalized operator $\mathcal{M}^{w_\psi}$ with $w_\psi\approx w^c$ with $c_t$ being the Retrace trace coefficient. Throughout the experiments, we measure the accuracy of the estimate as the relative error $\frac{|\hat{Q}-Q^\pi|}{|Q^\pi|}$, where $\hat{Q}$ is the estimate and $Q^\pi$ is the ground-truth Q-value.

 \paragraph{Results.} In Figure~\ref{fig:contraction}(a)-(b) shows that the increase in the number of actions or the horizon makes the evaluation more difficult: a large number of actions induces large variance in the estimation due to the increased ratio $\frac{\pi(a|x)}{\mu(a|x)}$; at the same time, long horizons require the propagation of values with more iterations. Overall, the marginalized operator converges faster than Retrace, which further outperforms the one-step operator. In Figure~\ref{fig:contraction}(c), we vary the off-policy level: all operators' performance increase as the problem interpolates from very off-policy to near on-policy. 
 
 While Figure~\ref{fig:contraction}(a)-(c) show the advantages of the marginalized operator, Figure~\ref{fig:contraction}(d) highlights potential limitations. As the noise level of the final reward increases, the marginalized operator and Retrace converge to a higher error rate than the one-step operator (similar observations are made in Figure~\ref{fig:contraction}(a)-(c)). We speculate that this is because as marginalized estimator and Retrace propagate downstream values more effectively, they also bootstrap noises faster. On the other hand, since the final reward is stochastic, we speculate that one-step operator's incremental back-up dampens the variance more significantly, leading to smaller asymptotic errors.
 The result implies that when there is much noise in the MDP, operators with short bootstrap horizons might be preferred.
 
 To compare Retrace and its marginalized counterpart, we vary the truncation level $\bar{c}$. Here, $\bar{c}$ controls the variance of the target values, as $\bar{c}=0$ reduces to the one-step operator while $\bar{c}=\infty$ reduces to full importance sampling. As shown in Figure~\ref{fig:contraction}(e), the performance of Retrace tends to be unstable when $\bar{c}$ is large; the marginalized operator, converges more stably though the asymptotic errors still increase as $\bar{c}$ increases.
 
\subsection{Open World}
Next, we consider the open world problem introduced in \citep{van2020expected}: it is a deterministic maze with $|\mathcal{X}|=n^2$ states with $n=10$. At each state, there are four actions $\mathcal{A}=\{\text{L},\text{U},\text{R},\text{D}\}$, each moving the agent to a neighboring state except when moving beyond the boundary, in which case the agent does not move. The agent always starts at the upper left corner. The reward is zero unless the agent transitions into the lower right corner terminal state, where $r=1$. 

We first consider off-policy evaluation. The agent estimates Q-function tables $\hat{Q}(x,a)$, but in Figure~\ref{fig:openworld} we color-code the value functions for all states computed as $\hat{V}(x)=\sum_a \pi(a|x)\hat{Q}(x,a)$. Here, the behavior policy $\mu$ is a uniformly random policy, while the target policy $\pi$ assigns all probability masses uniformly to $\{\text{D},\text{R}\}$. We compare three baselines: \textbf{(1) }one-step operator; \textbf{(2)} Retrace and \textbf{(3)} marginalized operator $\mathcal{M}^{w_\psi}$ with $w_\psi\approx w^c$. For further details and more results on policy optimization where off-policy evaluation is used as a subroutine, see Appendix~\ref{appendix:experiment}.

\paragraph{Results.} As observed in Figure~\ref{fig:openworld}, consistent with results in the chain MDP, the one-step operator propagates information rather slowly compared to the multi-step Retrace. When $\bar{c}=1$, the performance of Retrace and its marginalized counterpart is highly similar; however, when $\bar{c}=2$, Retrace becomes unstable. Indeed, moving from lower right to the upper left of the state space, the estimated values do not show any clear trend as in the case of $\bar{c}=1$, which implies potential divergence. On the other hand, the marginalized operator performs much more stably. All such observations imply that the marginalized operator might achieve an additional effect of variance reduction compared to Retrace. To better interpret the behavior of marginalized operators,we visualize the  TD weights $w_\psi$ in Appendix~\ref{appendix:experiment}. 
 \begin{figure}[h]
    \centering
    \includegraphics[width=.5\textwidth]{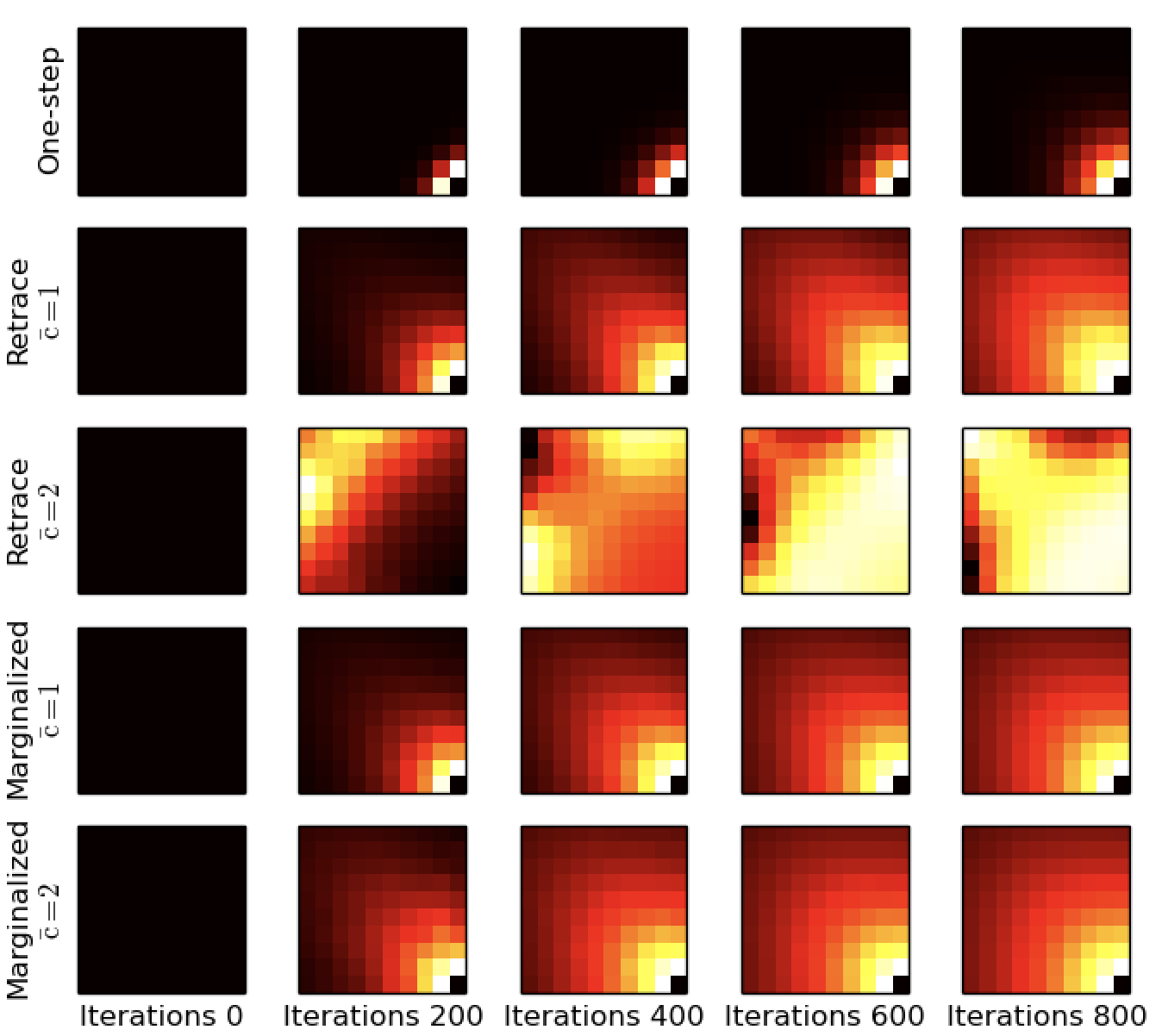}
    \caption{Comparison of operators on the Open World MDP. Each plot is averaged over $100$ runs. In each plot, moving from light yellow to red and further to black colors, the estimated values decrease. In Figure~\ref{fig:openworld}, going from the leftmost column to rightmost column, the number of iterations increases.}
    \label{fig:openworld}
\end{figure}

 \begin{figure}[h]
    \centering
    \subfigure[\textbf{Cheetah(D)} ]{\includegraphics[width=.22\textwidth]{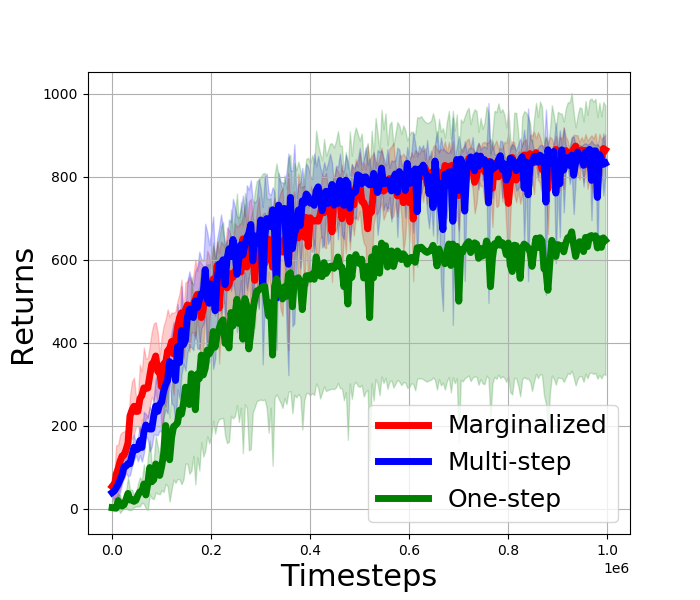}}
    \subfigure[\textbf{WalkerRun(D)}]{\includegraphics[width=.22\textwidth]{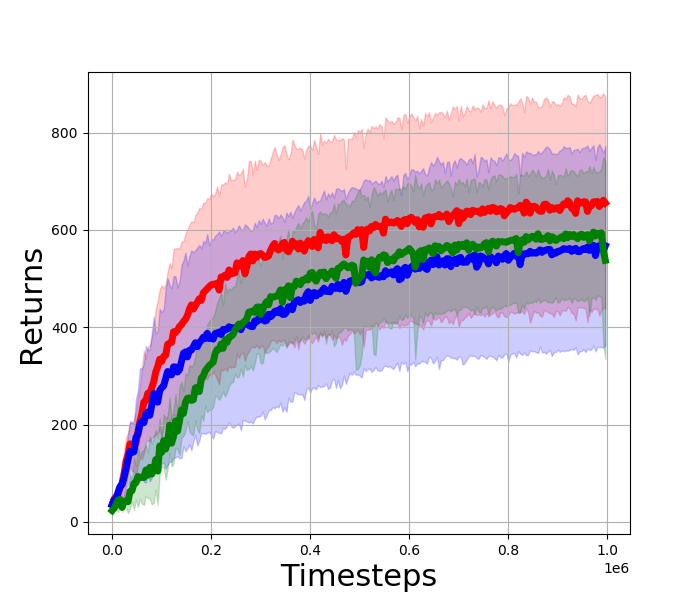}}
    \subfigure[\textbf{Cheetah(B)} ]{\includegraphics[width=.22\textwidth]{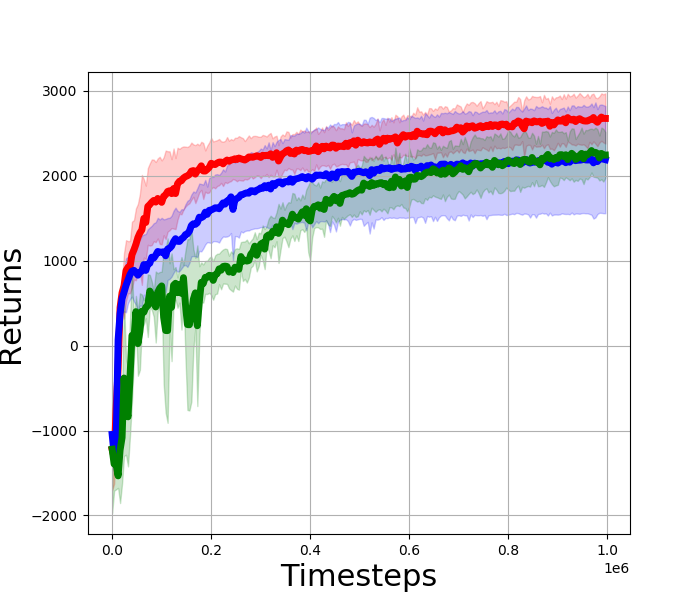}}
    \subfigure[\textbf{Ant(B)}]{\includegraphics[width=.22\textwidth]{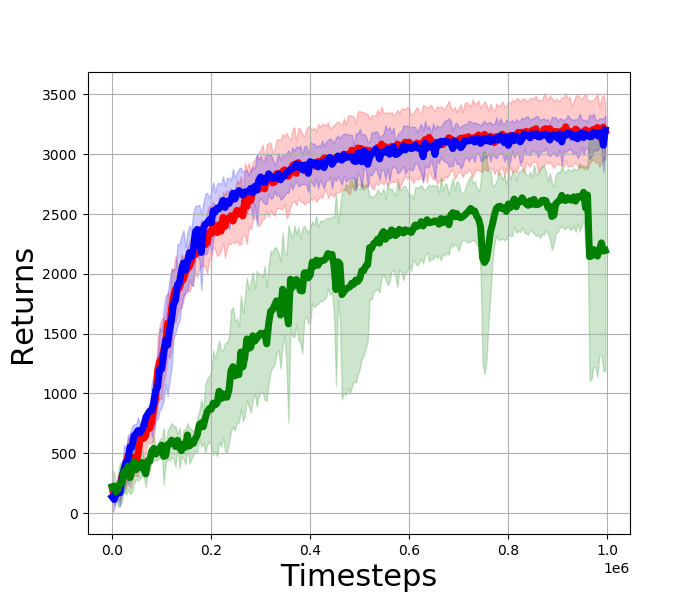}}
    \caption{Comparison of operators with deep RL implementations. Each curve is averaged over $5$ seeds. The x-axis shows the number of time steps at training time and y-axis shows the evaluation performance. The notation (D) and (B) denote the simulation backends of the testing environments. See Appendix~\ref{appendix:experiment} for further details on the experiments.}
    \label{fig:deeprl}
\end{figure}

\subsection{Deep RL experiments}
For high-dimensional state space (or high-dimensional action space), the estimation $w_\psi$ must be combined with more complex function approximation such as neural networks. We use simulated continuous control tasks as the test beds, and compare multi-step RL algorithms against the marginalized counterparts. We consider twin-delayed deep deterministic policy gradient (TD3) \citep{fujimoto2018addressing} as the base algorithm. TD3 implements a deterministic policy $\pi_\phi(x)$ and critic $Q_\theta(x,a))$, both parameterized by neural networks. The critic is updated by minimizing Bellman errors $\mathbb{E}\left\lbrack\left(Q_\theta(x,a)-Q_\text{target}(x,a)\right)^2\right\rbrack$ where $Q_\text{target}(x,a)$ is constructed by a few alternatives: one-step operator, multi-step operator and its equivalent marginalized operator. Take the one-step operator $\mathcal{T}^\pi$ as an example: given the transition $(x,a,r,x')$ the target is computed as a stochastic estimate to the exact back-up target
\begin{align*}
    Q_\text{target}(x,a) = r + Q_{\theta^-}(x',\pi(x')),
\end{align*}
where we can show $\mathbb{E}[Q_\text{target}(x,a)]=\mathcal{T}^\pi Q(x,a)$. Here, $\theta^-$ is the target network \citep{mnih2015human}. For multi-step operators, to construct stochastic estimates we follow the procedure in Section~\ref{sec:understand} to compute $\hat{\mathcal{R}}^c(x,a)$ and $\hat{\mathcal{M}}^{w}(x,a)$, which are unbiased to $\mathcal{R}^cQ(x,a)$ and $\mathcal{M}^wQ(x,a)$ respectively. In addition to the conventional actor-critic architecture, the marginalized operator also maintains an estimator $w_\psi$ parameterized by a neural network $\psi$ to approximate $w^c$ and applies the resulting operator $\mathcal{M}^{w_\psi}$. See Appendix~\ref{appendix:multistep} for further details on multi-step algorithms and Appendix~\ref{appendix:experiment} for more details on how to estimate $w^c$ with $w_\psi$ in practice with techniques introduced in Section~\ref{section:estimate}.

\paragraph{Results.} We show comparison in Figure~\ref{fig:deeprl},  where we evaluate algorithms over a subset of continuous control tasks \citep{brockman2016openai}. Overall, we find that multi-step updates might outperform or perform similarly as the one-step update, both in terms of learning speed and asymptotic performance; this is consistent with observations made in prior work on multi-step learning for value-based RL algorithms (see, e.g. \citep{kozuno2021revisiting}). Marginalized multi-step updates provide further marginal performance gains over the vanilla multi-step update in terms of the average performance. However, the variations across seeds are relatively large, indicating that the algorithm might be more unstable due to learning of TD weights.

 We discuss some challenges when combining marginalized operators in deep RL algorithms: training the density ratio estimator $w_\psi$ usually introduces computational overhead and potential instability to the overall algorithm. In large-scale distributed agents (see e.g., \citep{espeholt2018impala,kapturowski2018recurrent}), where the data throughput is large, it might not be worthwhile to incur the bias due to the marginalized estimation. It is of interest to further investigate how marginalized estimations can scale to such applications. 

\section{Conclusion}
We have proposed marginalized operators, a general class of off-policy evaluation operators. 
Marginalized operators bridge the conceptual gap between  multi-step operators and marginalized IS methods for off-policy evaluation. This provides a unified framework to reason about off-policy evaluation, with operator-based approaches and marginalized estimation methods as two seemingly separate yet compatible frameworks, and opens doors to new combinations of algorithmic techniques from both worlds. 

One interesting line of future work is to investigate the combination of marginalized operators with traditional multi-step operators, in a similar fashion as how marginalized IS combines with step-wise IS \citep{yuan2021sope}.

\paragraph{Acknowledgements.} The authors are thankful to constructive comments by anonymous reviewers. The authors thank Zheng Wen for reviewing prior drafts of this paper and are grateful for the supportive colleagues and great research environment at DeepMind.
\newpage
\bibliographystyle{apalike} 
\bibliography{main}

\clearpage

\onecolumn
\begin{appendix}

\section*{APPENDICES: Marginalized Operators for Off-Policy Reinforcement Learning}

\section{Proof of theoretical results}

\label{appendix:proof}

\propmarginalizedop*

\begin{proof}
We adopt the matrix notation to prove the result. See the proof for Proposition~\ref{prop:equivalent} for a detailed discussion on the matrix notation as well. By the definition of marginalized operators, define $W\in\mathbb{R}^{(\mathcal{X}\times\mathcal{A})\times(\mathcal{X}\times\mathcal{A})}$ as the tensor such that $W(x,a,x',a')=w_{x,a}^c(x',a')$.
Now, we can rewrite for any $Q\in\mathbb{R}^{\mathcal{X}\times\mathcal{A}}$,
\begin{align*}
    \mathcal{M}^w Q = Q + \left[(I-\gamma P^\mu)^{-1} \odot W \right] \left( R + \gamma P^\pi Q - Q\right),
\end{align*}
where $\odot$ denotes the element-wise product of two tensors with the same shape.
Let $D\coloneqq (I-\gamma P^\mu)^{-1}\odot W$.
For any $Q_1,Q_2$,
\begin{align*}
    \mathcal{M}^w Q_1 - \mathcal{M}^w Q_2 &= (\gamma DP^\pi -D + I)(Q_1-Q_2).
\end{align*}
Now, examine the $(x,a)$-th component of the vector $\mathcal{M}^w Q_1 - \mathcal{M}^w Q_2$. Through inspection, we can identify $d_{x,a}^w$ as the $(x,a)$-th row of $D$ scaled by $(1-\gamma)$. This implies 
\begin{align*}
     \mathcal{M}^w Q_1(x,a) - \mathcal{M}^w Q_2(x,a) = (1-\gamma)^{-1}\left((1-\gamma)\delta_{x,a}+\gamma (P^\pi)^Td_{x,a}^w - d_{x,a}^w\right)^T (Q_1-Q_2).
\end{align*}
If we define $E_{x,a}^w\coloneqq(1-\gamma)\delta_{x,a}+\gamma (P^\pi)^Td_{x,a}^w - d_{x,a}^w$, the above rewrites as 
\begin{align*}
    \mathcal{M}^w Q_1(x,a) - \mathcal{M}^w Q_2(x,a)=  (1-\gamma)^{-1}E_{x,a}^w (Q_1-Q_2).
\end{align*}
Hence the contraction rate of the operator is $\max_{x,a}(1-\gamma)^{-1}\left\lVert E_{x,a}^w \right\rVert_1$.

\end{proof}

\propequivalent*

\begin{proof}
We start by assuming $d_{x,a}^\mu(x',a')>0$ for all $(x,a),(x',a')$. We introduce matrix notations for the marginalized operator. For TD weights $w$, let $W$ be a matrix such that $W(x,a,x',a')=w_{x,a}(x',a')$. For any two matrices $A,B$ of the same shape, let $A \odot B$ be the element-wise product. Let $R\in\mathbb{R}^{\mathcal{X}\times\mathcal{A}}$ be the expected reward vector such that $R(x,a)=\bar{r}(x,a)$. By the definition of marginalized operators, we rewrite
\begin{align*}
    \mathcal{M}^w Q = Q + \left[(I-\gamma P^\mu)^{-1} \odot W \right] \left( R + \gamma P^\pi Q - Q\right).
\end{align*}
We first assume that the multi-step operator adopts Markovian step-wise traces. Let $P^{c\mu}$ be the transition matrix defined by the sub-probability measure $c\mu$ such that $P^{c\mu}(x,a,x',a')=p(x'|x,a)\mu(a'|x')c(x',a')$. We can write
\citep{munos2016safe}
\begin{align*}
    \mathcal{R}^c Q = Q + (I-\gamma P^{c\mu})^{-1} \left(R + \gamma P^\pi Q - Q\right).
\end{align*}
By letting $\mathcal{M}^w=\mathcal{R}^c$, we can see the following is a solution to $w$ 
\begin{align}
    W = (I-\gamma P^{c\mu})^{-1} / (I-\gamma P^\mu)^{-1}.
    \label{eq:matrix-equality}
\end{align}
Here, for two matrices $A,B$ of the same shape, we define $A/B$ to be the element-wise division, where it is required that all entries of $B$ are strictly positive. Note that $(I-\gamma P^{c\mu})^{-1}=\sum_{t=0}^\infty \left(\gamma P^{c\mu}\right)^t$ and $(I-\gamma P^\mu)^{-1}=\sum_{t=0}^\infty \left(\gamma P^{\mu}\right)^t$. This implies that the $(x,a,x',a')$ component of $(I-\gamma P^\mu)^{-1}$ is $(1-\gamma)^{-1}d_{x,a}^\mu(x',a')$, and the $(x,a,x',a')$ component of $(I-\gamma P^{c\mu})^{-1}$ is accordingly
\begin{align*}
     \frac{1-\gamma}{d_{x,a}^\mu(x^\prime,a^\prime)} \mathbb{E}_\mu\left\lbrack \sum_{t\geq 0} \gamma^t \left(\Pi_{1\leq s\leq t}c_s\right) \mathbb{I}[ x_t=x^\prime,a_t=a^\prime] \; \middle| \; x_0=x,a_0=a \right\rbrack.
\end{align*}
By reading off components from the matrix equality Eqn~\eqref{eq:matrix-equality}, we arrive at the desired result. 

When the traces are non-Markovian, the proof can be extended naturally. Let $c_t \in \mathbb{R}^{\mathcal{X}\times\mathcal{A}}$ be a vector such that $c_t(x,a)$ defines the step-wise trace at time $t$ after starting with $(x,a)$. The multi-step operator can be written as \begin{align}
    \mathcal{R}^c Q = Q + \sum_{t=0}^\infty  \left(\Pi_{0\leq s\leq t} P^{c_t\mu}\right) \left(R + \gamma P^\pi Q - Q\right).\label{eq:general-retrace}
\end{align}
We then arrive at the following sufficient condition for $\mathcal{M}^w=\mathcal{R}^c$ 
\begin{align*}
    W = \sum_{t=0}^\infty  \left(\Pi_{0\leq s\leq t}P^{c_t\mu}\right) / (I-\gamma P^\mu)^{-1}.
\end{align*}
By reading off components of both sides, we arrive at the desired conclusion.

Now in case for some $(x,a,x',a')$, $d_{x,a}^\mu(x',a')=0$, we can safely set $w_{x,a}^c(x',a')=0$. This is because $d_{x,a}^\mu(x',a')=0$ implies that there is zero probability that the agent arrives in $(x',a')$ starting from $(x,a)$, which means Bellman errors starting from $(x',a')$ are never computed as part of expectation which defines the operator. 

\paragraph{Technical conditions for the summation in Eqn~\eqref{eq:marginalized-trace}.} 
It is clear that there exists some step-wise traces $c_t$ such that the summation in Eqn~\eqref{eq:marginalized-trace} does not converge, e.g., $c_t=\frac{1}{\gamma}$. We impose a condition: (\textbf{C.1}) The step-wise traces $c_t$ should be such that $\mathcal{R}^c Q$ is finite under the definition in Eqn~\ref{eq:general-retrace}. Naturally, (\textbf{C.1}) implies that $\sum_{t=0}^\infty  \left(\Pi_{0\leq s\leq t}P^{c_t\mu}\right)$ is finite element-wise, which further implies that the infinite sum $\mathbb{E}_\mu\left\lbrack \sum_{t\geq 0} \gamma^t \left(\Pi_{1\leq s\leq t}c_s\right) \mathbb{I}[ x_t=x^\prime,a_t=a^\prime] \; \middle| \; x_0=x,a_0=a \right\rbrack$ is finite for all $(x,a),(x',a')$. Note that (\textbf{C.1}) is very weak and is valid for all situations of interest to us.

\end{proof}

\corospace*

\begin{proof}
Given any step-wise traces $c_t$ (Markovian or non-Markovian), we can compute corresponding marginalzied traces $w$ via Eq~\eqref{eq:marginalized-trace}. Then $\mathcal{R}^c=\mathcal{M}^w$ per Proposition~\ref{prop:equivalent}. This implies the desired result in the corollary.
\end{proof}

\propmore*
\begin{proof}

We start with some clarifications on notations. The TD weights $c_t$ could be either Markovian or non-Markovian. In the latter case, we require that $c_t$ is measurable w.r.t. $(x_s,a_s)_{s\leq t}$. Given a tuple of MDP, policy and discount factor $T=(r,p,\pi,\mu,\gamma)$, Note that here $c_t$ could be Markovian or non-Markovian. Let $\mathcal{C}_\text{markov}(T) \in \mathbb{R}^{\left(\mathcal{X}\times\mathcal{A}\right)\times \left(\mathcal{X}\times\mathcal{A}\right) }$ be the set of Markovian traces such that $\mathcal{R}^c$ is contractive; let $\mathcal{C}_\text{non-markov}(T)\in \left(\mathbb{R}^{\mathcal{X}\times\mathcal{A}}\right)^H$ be the set of non-Markovian traces such that $\mathcal{R}^c$ is contractive, where $H$ is horizon of the Markov chain induced by $\pi$ starting from any state-action pair. In general, we consider $H=\infty$. As such, for any $T$, $\mathcal{C}(T)=\mathcal{C}_\text{markov}(T)\cup \mathcal{C}_\text{non-markov}(T)$. Finally, let $\mathcal{W}(T)$ be the set of TD weights such that for any $w\in\mathcal{W}(T)$, any $\mathcal{M}^w\in\mathcal{W}(T)$ is contractive.

Per Proposition~\ref{prop:equivalent}, we can start with any $c\in\mathcal{C}(T)$ and project it into a $w\in\mathcal{W}(T)$. For convenience of the discussion, we denote such a projection as $f_{c\rightarrow w}^{T}$, where the $T$ denotes that this projection generally depends on $T$ (e.g., the expectation defined in Eqn~\eqref{eq:marginalized-trace} is computed with respect to the dynamics $p$). Formally, we can write $f_{c\rightarrow w}^{T}: \mathcal{C}(T)\mapsto \mathcal{W}(T)$.

We state a few important properties of $f_{c\rightarrow w}^{T}$ as lemmas.
\begin{restatable}{lemma}{lemma1}
\label{lemma:1}
When constrained $f_{c\rightarrow w}^{T}$ to Markovian traces, let the constrained mapping be $f_{c\rightarrow w}^{\mathcal{C}_\text{markov},T}: \mathcal{C}_\text{markov}(T) \mapsto W(T)$. There exists tuples $T$ such that $f_{c\rightarrow w}^{\mathcal{C}_\text{markov},T}$ is not surjective.
\end{restatable}
\begin{proof}
We prove by constructing a counterexample where for some $T$, there exists a $w \in \mathcal{W}(T)$ that cannot be obtained by first picking a Markovian trace $c\in\mathcal{C}_\text{markov}$ and then project it through $f_{c\rightarrow c}^{\mathcal{C}_\text{markov},T}$.

Consider a deterministic chain MDP with $N$ states $\{x_i\}_{i=1}^N$. All first $N-1$ states transition deterministically to the next state on the right. The last (rightmost) state is absorbing. Assume also $\pi=\mu$ to be both deterministic policy. Consider the TD weights $w^\ast$ such that its $(x,a,x',a')$ component is $w_{x,a}(x',a')=\frac{\delta_{x'=x,a'=a}}{d_{x,a}^\mu(x',a')}$. In this case, the operator $\mathcal{M}^{w^\ast}$ is exactly the one-step TD operator. Starting from state $x_i,1\leq i\leq N-1$, the marginalized operator is
\begin{align*}
    \mathcal{M}^w Q(x_i,a_i)= Q(x_i,a) + \left(r_i + \gamma Q\left(x_{i+1},\pi(x_{i+1})\right) - Q(x_i,a)\right).
\end{align*}
The step-wise operator is
\begin{align*}
    \mathcal{R}^c Q(x_i,a_i) =
    Q(x_i,a_i) + \sum_{i\leq j\leq N-1} \gamma^{j-i} \left(c_{i+1}...c_j\right)  \left(r_j + \gamma  Q\left(x_{j+1},\pi(x_{j+1})\right) - Q(x_j,a_j) - Q(x_j)  \right) + F(N),
\end{align*}
where $F(N)$ is some function of the last state. Now, we find $c$ such that $\mathcal{M}^{w^\ast}=\mathcal{R}^c$. By matching coefficients of the term $Q(x_{i+1},a_{i+1})$, it is necessary that $c(x_i,a_i)=1$. However, by setting $c(x_i,a_i)=1$, $\mathcal{R}^c\neq \mathcal{M}^{w^\ast}$. In other words, there does not exist a Markovian trace $c \in \mathbb{R}^{\mathcal{X}\times\mathcal{A}}$ such that $f_{c\rightarrow w}(c) = w^\ast$. This implies that under this setup, the mapping is not surjective.
\end{proof}

\begin{restatable}{lemma}{lemma2}
\label{lemma:2}
Let $W^+(T)=\{w\in W(T),w>0\}\subset W(T)$. For any $T$,  $f_{c\rightarrow w}^{T}$ is surjective to $W^+(T)$.
\end{restatable}
\begin{proof}
Intuitively, for those TD weights $w$ that could not be realized by Markovian step-wise traces, we need to construct non-Markovian step-wise traces $c_t$ to construct them, such that $f_{c\rightarrow w}^{T}(c)=w$.

We construct non-Markovian step-wise traces as follows. Given $w \in \mathcal{W}^+(T)$. starting from $(x,a)$, the step-wise coefficient at time $t\geq 0$ is computed as
\begin{align*}
    c_t \coloneqq \frac{w_{x,a}(x_t,a_t)}{w_{x,a}(x_{t-1},a_{t-1})},
\end{align*}
where we define $w_{x,a}(x_t,a_t)=1$ for $t=-1$. We can show that by such a construction, $(\Pi_{1\leq s\leq t} c_s) = w_{x,a}(x_t,a_t)$ and as such 
\begin{align*}
    f_{c\rightarrow w}(T)(c) = w.
\end{align*}
\end{proof}

\begin{restatable}{lemma}{lemma3}
\label{lemma:3}
There exists $T$, such that $f_{c\rightarrow w}(T)$ is \textbf{not} surjective to $W(T)$.
\end{restatable}
\begin{proof}

We construct a counterexample of $T$. In this case, we seek TD weights $w \in W(r,p,\pi,\mu,\gamma)$ such that we cannot find $c\in \mathcal{C}(T)$ such that $f_{c\rightarrow w}(T)(c)=w$. Notably, in this case, $\mathcal{C}(T)$ should contain all step-wise traces, both Markovian and non-Markovian ones.

Consider a deterministic chain MDP with $|\mathcal{X}|=N=5$ states $\{x_i\}_{i=1}^N$ and $|\mathcal{A}|=2$ actions $\{a_i\}_{i=1}^2$. All first $N-1$ states transition deterministically to the next state on the right. The last (rightmost) state is absorbing. Assume that $\pi=\mu$ are both uniformly random. Finally, let $\gamma=0.8$.

Consider the contraction property of $\mathcal{M}^w$ starting from the state $(x_1,a_1)$. We can show that by defining $d_{x_1,a_1}(x',a')=0$ except
\begin{align*}
    d_{x_1,a_1}(x_1,a_1)=0.2, d_{x_1,a_1}(x_3,a_3)=0.01.
\end{align*}
Then we set $w_{x_1,a_1}=\frac{d_{x_1,a_1}}{d_{x_1,a_1}^\mu}$ (element-wise division). We can show that 
\begin{align*}
   \left| \mathcal{M}^w Q_1 - \mathcal{M}^w Q_2 \right| (x_1,a_1) \leq 0.89 \left\lVert Q_1 - Q_2 \right\rVert_\infty.
\end{align*}
This implies that the resulting operator $\mathcal{M}^w$ is contractive for the pair $(x_1,a_1)$. We can complete the definition of $w$ for other state-action pairs $(x,a)$ by specifying $w_{x,a}$ properly. Concretely, as an example, we might set $w_{x,a}^{x',a'}=\frac{\delta_{x'=x,a'=a}}{d_{x,a}^\mu(x',a')}$ so that $\left| \mathcal{M}^w Q_1 -  \mathcal{M}^w Q_2 \right|(x',a') \leq \gamma\left\lVert Q_1 - Q_2 \right\rVert_\infty = 0.8\left\lVert Q_1 - Q_2 \right\rVert_\infty $ for any $(x',a')\neq (x_1,a_1)$. Overall, the operator is contractive
\begin{align*}
    \left\lVert \mathcal{M}^w Q_1 -  \mathcal{M}^w Q_2 \right\lVert_\infty \leq 0.89 \left\lVert Q_1 - Q_2 \right\rVert_\infty.
\end{align*}

Now, we argue why this particular choice of $w_{x_1,a_1}$ cannot be realized by any step-wise traces. Note that since by construction, $d_{x_2,a}=0,\forall a\in\{a_1,a_2\}$. This implies that starting from $(x_1,a_1)$, if we seek any step-wise traces which are equivalent to $d_{x_1,a_1}$, they must cut the traces at $(x_2,a)$. A direct consequence of this result is that $c(x_2,a)=0$ for both Markovian or non-Markovian traces. However, since the traces are multiplicative, this further means that the cumulative product of traces at $(x_3,a)$ would be zero. This does not replicate the behavior of $d_{x_1,a_1}$, whose entry at $(x_3,a_3)$ is constructed to be $0.01>0$.

To summarize, the above example shows that under this particular set of $T$, there exists a $w$ that cannot be realized by any step-wise traces through the mapping $f_{c\rightarrow w}(T)$. Hence the result is concluded.
\end{proof}

\begin{restatable}{lemma}{lemma4}
\label{lemma:4}
There exists $T$, such that $f_{c\rightarrow w}(T)$ is surjective to $W(T)$.
\end{restatable}
\begin{proof}
Consider a special case where we have $|\mathcal{X}|=2$ states and $|\mathcal{A}|=1$ action. Let $x_1,x_2$ be the states and $a_1$ the single action. Assume also all rewards are deterministic. As such, the policy $\pi,\mu$ are trivial as $\pi(a_1|x)=\mu(a_1|x)=1,\forall x$. The transition matrix is
\begin{align*}
P^\pi = 
    \begin{pmatrix}
    0 & 1\\ 1 & 0
    \end{pmatrix}.
\end{align*}
With the above setup, consider any marginalized trace at $(x_1,a_1)$, $w_{x_1,a_1}\in W(r,p,\pi,\mu,\gamma)$. Note that $w_{x_1,a_1} \in\mathbb{R}^2$. Let $c_t,t\geq 1$ be the non-Markovian step-wise trace starting from $(x_1,a_1)$. Define the one-step Bellman errors $\Delta_1\coloneqq r_1 + \gamma Q(x_2,a_1)-Q(x_1,a_1), \Delta_2 \coloneqq r_2 + \gamma Q(x_1,a_1)-Q(x_2,a_1)$.

The marginalized operator evaluated at $(x_1,a_1)$ is
\begin{align*}
\mathcal{M}^w Q(x_1,a_1) = Q(x_1,a_1) + \left(1+\gamma^2+\gamma^4+...\right) w_{x_1,a_1}^{x_1,a_1} \Delta_1 + \left(\gamma + \gamma^3 + \gamma^5 + ...\right) w_{x_1,a_1}^{x_2,a_1} \Delta_2.
\end{align*}
The step-wise operator is \begin{align*}
    \mathcal{R}^c Q(x_1,a_1) = Q(x_1,a_1) + \left(1 + \gamma^2 c_1c_2 + \gamma^4 c_1c_2c_3c_4 + ...\right) \Delta_1 + \left(\gamma c_1 + \gamma^3 c_1c_2c_3 + ...\right)\Delta_2.
\end{align*}
We can identify the following solution $c$ to satisfy the equality $\mathcal{M}^wQ(x_1,a_1)=\mathcal{R}^c(x_1,a_1)$.
\begin{align*}
    c_1=1, c_2=\frac{A-1}{\gamma^2},c_3=B-\gamma-\gamma(A-1),c_4=c_5=...=0,
\end{align*}
where $A=\left(1+\gamma^2+\gamma^4\right)w_{x_1,a_1}(x_1,a_1),B=\left(\gamma+\gamma^3+\gamma^5+...\right)w_{x_1,a_1}(x_2,a_1)$. Note that the solution always exists regardless of $w_{x_1,a_1}$. In a similar way, we can solve for non-Markovian traces for $w_{x_2,a_1}$ as well. We conclude for any $w$, there exists non-Markovian traces $c$ such that $f_{c\rightarrow w}(r,p,\pi,\mu,\gamma)(c)=w$ for the above $(r,p,\pi,\mu,\gamma)$.
\end{proof}

The above lemmas characterize the space of $\{\mathcal{R}^c,c\in \mathcal{C}(T)\}$ relative to $\{\mathcal{M}^w,w \in\mathcal{W}(T) \}$. From Lemma~\ref{lemma:3} we conclude  case (i) of the proposition; from Lemma~\ref{lemma:4}, we conclude the case (ii) of the proposition.

\end{proof}

\propcond*
\begin{proof}
The definition of $w^c$ could rewrite as 
\begin{align*}
    w_{x,a}^c(x',a') \cdot d_{x,a}^\mu(x',y') = \mathbb{E}_{\mu,\tau}\left\lbrack (\Pi_{1\leq s\leq \tau} \mathbb{I}[x_\tau=x',a_\tau=a']\right\rbrack
\end{align*}
As such, we expand the RHS of the above
\begin{align}
    w_{x,a}^c(x',a')  \cdot d_{x,a}^\mu(x',y') &= \mathbb{E}_{\mu,\tau}\left\lbrack (\Pi_{1\leq s\leq \tau} \mathbb{I}[x_\tau=x',a_\tau=a']\right\rbrack \nonumber \\
    &= \mathbb{E}_{\mu,\tau}\left\lbrack( \Pi_{1\leq s\leq \tau} c_s) \; \middle| \; x_\tau=x',a_\tau=a'  \right\rbrack \nonumber\\ &\times P_\mu(x'_\tau=x',a_\tau=a|x_0=x,a_0=a).
\end{align}
Also note that $d_{x,a}^\mu(x',a')\coloneqq (1-\gamma) \sum_{t\geq0} \gamma^t P_\mu(x_t=x',a_t=a'|x_0=x,a_0=a)=P_\mu(x_\tau=x',a_\tau=a'|x_0=x,a_0=a)$, which cancel on both sides of the equation. Hence we conclude the equality.
\end{proof}

\corocondoperator*
\begin{proof}
With Proposition~\ref{prop:conditional-is}, we have $w^c(x_\tau,a_\tau)=\mathbb{E}_{\mu,\tau}\left\lbrack( \Pi_{1\leq s\leq \tau} c_s) \; \middle| \; x_\tau,a_\tau  \right\rbrack$. Further,
\begin{align*}
w(x_\tau,a_\tau) \hat{\Delta}_\tau^\pi &= 
     \mathbb{E}_{\mu,\tau}\left\lbrack( \Pi_{1\leq s\leq \tau} c_s) \; \middle| \; x_\tau,a_\tau  \right\rbrack \hat{\Delta}_\tau^\pi \\ \nonumber
     &=  \mathbb{E}_{\mu,\tau}\left\lbrack( \Pi_{1\leq s\leq \tau} c_s) \delta_\tau^\pi \; \middle| \; x_\tau,a_\tau  \right\rbrack \nonumber
\end{align*}
Note that since the transitions are deterministic $\hat{\Delta}_\tau^\pi$ is a measurable function of $(x_\tau,a_\tau)$ and could be taken out of the expectation. Then with the tower property of variance $\mathbb{V}\left\lbrack X  \right\rbrack \geq \mathbb{V}\left\lbrack \mathbb{E} \left\lbrack X \; \middle| \;  Y \right\rbrack \right\rbrack$, by letting $X=(\Pi_{1\leq s\leq \tau} c_s) \hat{\Delta}_\tau^\pi$ and $Y=(x_\tau,a_\tau)$ we conclude the result.
\end{proof}

\begin{restatable}{proposition}{propbellman}\label{prop:trace-bellman}
For any step-wise trace coefficient $c_t$, its equivalent TD weights $w^c$ and $d_{x,a}^{w^c} \coloneqq d_{x,a}^\mu\odot w_{x,a}^c$,
\begin{align}
    d_{x,a}^{w^c}=(1-\gamma)\delta_{x,a} + \gamma (P^{\tilde{\pi}})^T d_{x,a}^{w^c},\label{eq:trace-bellman}
\end{align}
where $\tilde{\pi}(a|x)=\pi(a|x)c(x,a)$ and $\tilde{\pi}(a|x)\coloneqq \mu(a|x)c(x,a)$ is a non-negative measure for any $0\leq c(x,a)\leq \frac{\pi(a|x)}{\mu(a|x)}$.
\end{restatable}

\begin{proof}
We show the Bellman equation directly from the definition of $d_{x,a}^{w^c}(x',a')$. In the following, we always condition on $x_0=x,a_0=a$ inside expectations. For the simplicity of notations, we drop this conditioner by default. It is clear that by construction,
\begin{align*}
    d_{x,a}^{w^c}(x',a') = (1-\gamma) \mathbb{E}_{\mu}\left\lbrack \sum_{t\geq 0} \gamma^t (\Pi_{s=1}^t c_s) \mathbb{I}[x_t=x',a_t=a'] \right\rbrack
\end{align*}
We rewrite the above into the following
\begin{align*}
    d_{x,a}^{w^c}(x',a') &= (1-\gamma) \mathbb{I}[x_0=x',a_0=a'] + \mathbb{E}_{\mu}\left\lbrack \sum_{t\geq 1} \gamma^t \left(\Pi_{s=1}^t c_s\right) \mathbb{I}[x_t=x',a_t=a']\right\rbrack\\
    &=(1-\gamma) \mathbb{I}[x_0=x',a_0=a'] + \gamma\mathbb{E}_{\mu}\left\lbrack \sum_{u\geq 0} \gamma^u \left(\Pi_{s=1}^u  c_s\right) c_{u+1} \mathbb{I}[x_{u+1}=x',a_{u+1}=a']\right\rbrack,
\end{align*}
where in the second equality we apply the transformation $u=t-1$. Now, let $h_u\coloneqq \{x_0=x,a_0,...x_u,a_u\}$ denote the sequence of random variables until time $u$. For each term in the summation, for any given $u\geq 0$,
\begin{align*}
    \mathbb{E}_{\mu}\left\lbrack  \gamma^u (\Pi_{s=1}^u  c_s) c_{u+1} \mathbb{I}[x_{u+1}=x',a_{u+1}=a']\right\rbrack  &= \sum_{y\in\mathcal{X},b\in\mathcal{A}}  \mathbb{E}_{\mu}\left\lbrack  \gamma^u (\Pi_{s=1}^u  c_s) c_{u+1} \mathbb{I}[x_{u+1}=x',a_{u+1}=a']\mathbb{I}[x_u=y,a_u=b]\right\rbrack \\
    &= \sum_{y\in\mathcal{X},b\in\mathcal{A}}  \mathbb{E}_\mu\left\lbrack\mathbb{E}_{\mu}\left\lbrack  \gamma^u (\Pi_{s=1}^u  c_s) c_{u+1} \mathbb{I}[x_{u+1}=x',a_{u+1}=a']\mathbb{I}[x_u=y,a_u=b] \;\middle|\; h_u\right\rbrack \right\rbrack \\
    &= \sum_{y\in\mathcal{X},b\in\mathcal{A}}  \mathbb{E}_\mu\left\lbrack\mathbb  \gamma^u (\Pi_{s=1}^u  c_s)  \mathbb{I}[x_u=y,a_u=b] P^{\tilde{\pi}}(x_u,a_u,x',a')\right\rbrack  \\
    &=  \mathbb{E}_\mu\left\lbrack\mathbb  \gamma^u (\Pi_{s=1}^u  c_s)  \mathbb{I}[x_u=y,a_u=b] P^{\tilde{\pi}}(y,b,x',a')\right\rbrack  .
\end{align*}
In the above, we have used the equality,
\begin{align*}
    \mathbb{E}_\mu\left\lbrack c_{u+1} \mathbb{I}[x_{u+1}=x',a_{u+1}=a']\;\middle|\; h_u \right\rbrack = \mathbb{E}_\mu\left\lbrack c_{u+1} \mathbb{I}[x_{u+1}=x',a_{u+1}=a']\;\middle|\; x_u,a_u \right\rbrack = P^{\tilde{\pi}}(x_u,a_u,x',a'),
\end{align*}
which derives from the definition of the transition matrix. Finally, we sum up over the time step $k$ to yield the final fixed point equation,
\begin{align*}
     d_{x,a}^{w^c}(x',a') &= (1-\gamma) \mathbb{I}[x_0=x',a_0=a'] + \gamma \sum_{y\in\mathcal{X},b\in\mathcal{A}}  d_{x,a}^{w^c}(y,b)P^{\tilde{\pi}}(y,b,x',a').
\end{align*}
By rewriting the above equation into the matrix form, we conclude the proof.

\paragraph{Alternative proof by matrix notations.} We can derive much simpler alternative proof with matrix notations. Let $d^{w^c}\in\mathbb{R}^{\left(\mathcal{X}\times\mathcal{A}\right)\times \left(\mathcal{X}\times\mathcal{A}\right)}$ be a matrix such that $d^{w^c}(x,a,x',a')=d_{x,a}^{w^c}(x',a')$. Also define the visitation distribution matrix $d^\mu=(1-\gamma) (I-\gamma P^\mu)^{-1}$. Recall that from the proof of Proposition~\ref{prop:equivalent}, in matrix form,
\begin{align*}
    W = (I-\gamma P^{c\mu})^{-1} / (I-\gamma P^\mu)^{-1}.
\end{align*}
Then by construction,
\begin{align*}
    d^{w^c} = (1-\gamma) W \odot d^\mu = (1-\gamma) (I-\gamma P^{c\mu})^{-1} = (1-\gamma) \sum_{t=0}^\infty (\gamma P^{c\mu})^t.
\end{align*}
Then naturally, $d^{w^c}$ satisfies the following Bellman equations,
\begin{align*}
    d^{w^c} = (1-\gamma) + \gamma P^{c\mu} d^{w^c}.
\end{align*}
When indexing the row at $(x,a)$, we arrive at the desired result.

\end{proof}

\propestimate*
\begin{proof}
Define $d_\psi \coloneqq w_\psi \odot d_{x,a}^\mu$.
By construction of the objective $L(\mathbf{q},w_\psi)$, we can rewrite the objective as an inner product, 
\begin{align}
    L(\mathbf{q},w_\psi) = \mathbf{q}^T [(1-\gamma)\delta_{x,a} + \gamma (P^{\tilde{\pi}})^T d_\psi -d_\psi],
    \label{eq:inner-product}
\end{align}
where $\tilde{\pi}(a|x)=\mu(a|x)c(x,a)$ is a sub-probability measure. Per results in Proposition~\ref{prop:trace-bellman}, the objective satisfies the following equation when the second argument is $w^c$
\begin{align*}
    L(\mathbf{q},w^c)=0.
\end{align*}
Hence, we can rewrite Eqn~\eqref{eq:inner-product} as the following
\begin{align*}
    L(\mathbf{q},w_\psi) &= L(\mathbf{q},w_\psi)-L(\mathbf{q},w^c)\\ &
    = \mathbf{q}^T [\gamma(P^{\tilde{\pi}})^T-I](d_\psi-w^c).
\end{align*}
Rewriting the product of matrix and vectors into expectations,
\begin{align*}
    L(\mathbf{q},w_\psi) &= \mathbb{E}_{(x',a;)\sim d_{x,a}^\mu}\left\lbrack (w_\psi(x',a')-w(x',a')) (\Pi q)(x',a')\right\rbrack,
\end{align*}
where $(\Pi q)(x',a')\coloneqq \gamma \mathbb{E}_{a''\sim\tilde{\pi}(\cdot|x'')}\left\lbrack q(x'',a'')\right\rbrack-q(x',a')$ where $x''\sim p(\cdot|x',a')$. Interestingly, here $(\Pi q)(x',a')$ could be interpreted as a reward such that if policy $\tilde{\pi}$ is executed, the Q-function would be $q(x',a')$. Following the techniques of \citep{liu2018breaking}, it is straightforward to show that when $q(x',a',x^\ast,a^\ast)=\sum_{t\geq0} \gamma^t T_{\tilde{\pi}}^t (x',a'|x,a)$, we have $(\Pi q)(x',a')=\delta(x'=x^\ast,a'=a^\ast)$. Here, importantly, because $\tilde{\pi}$ is a sub-probability measure, $T_{\tilde{\pi}}^t$ exists and $\sum_{t\geq0}T_{\tilde{\pi}}^t$ converges. As a result, with this choice of $\mathbf{q}(x^\ast,a^\ast)$, we have $L(\pm \mathbf{q}(x^\ast,a^\ast),w_\psi)=\pm (w_\psi(x^\ast,a^\ast)-w(x^\ast,a^\ast))$. Then it follows that when $\{
\pm q(x',a',x^\ast,a^\ast),\forall (x,a)\}\in \mathcal{Q}$, the error $|w_\psi(x^\ast,a^\ast)-w(x^\ast,a^\ast)|$ is upper bounded by $\max_{\mathbf{q}\in\mathcal{Q}}L(\mathbf{q},w_\psi)$.
\end{proof}

\propcontraction*
\begin{proof}
Assume $\mathcal{Q}_b\subset\mathcal{Q}$. Based on Eqn~\eqref{eq:inner-product}, we deduce the following
\begin{align*}
    \max_{\mathbf{q}\in\mathcal{Q}} L(\mathbf{q},w_\psi) = \max_{\mathbf{q}\in\mathcal{Q}} \mathbf{q}^T [(1-\gamma)\delta_{x,a}+\gamma (P^\pi)^T w_\psi- w_\psi] = \sum_{x',a'}\left|(1-\gamma)\delta_{x,a}+\gamma (P^\pi)^T w_\psi- w_\psi\right|(x',a') = \eta_{w_\psi}.
\end{align*}
Here, the maximizer $\mathbf{q}^\ast \in \mathcal{Q}_b$ is \begin{align*}\mathbf{q}^\ast(x,a)=\text{sign}\left((1-\gamma)\delta_{x,a}+\gamma (P^\pi)^T w_\psi- w_\psi] \right),
\end{align*}
where $\text{sign}(x)$ is the element-wise sign function.
\end{proof}

\proplp*
\begin{proof}
Let $Q_t(x,a)$ be the set of LP objectives at iteration $t$, and assign them to the objective coefficients of LPs at iteration $t+1$. This operation is defined through an equivalent operator $\mathcal{R}$. Recall that we abuse notations and denote $Q_{t+1},Q_t\in\mathbb{R}^{\mathcal{X}\times\mathcal{A}}$ as vector Q-functions. Let $d^\ast$ be the optimal solution to the $\text{LP}^{(t)}(x,a)$, then by construction
\begin{align*}
    Q_{t+1}(x,a) = \delta_{x,a}^T Q_t + (1-\gamma)^{-1}(d^\ast)^T (R + \gamma P^\pi Q_t - Q_t).
\end{align*}
Then recall that the constraints in Eqn~\eqref{eq:eval-duallp-augmented-relaxed} imply that 
\begin{align*}
    \left\lVert (1-\gamma)\delta_{x,a} + \gamma(P^\pi)^T d^\ast - d^\ast \right\rVert_1 \leq (1-\gamma)\eta.
\end{align*}
This implies that the iteration $Q_{t+1}\leftarrow Q_t \equiv \delta_{x,a}^T Q_t + (1-\gamma)^{-1}(d^\ast)^T (R + \gamma P^\pi Q_t - Q_t)$ is contractive. In addition, the fixed point of this process is $Q^\pi$. We then conclude the desired result. 

\paragraph{Discussion on more general results.} The above proof relies on the important fact that the feasible set defined in $\text{LP}(x,a)$ in Eqn~\eqref{eq:eval-duallp-augmented-relaxed} corresponds to a set of $d$ such that the iteration process is contractive. Hence, if we choose any arbitrary element $\tilde{d}\in \mathcal{D}_{x,a}$, and define
\begin{align*}
    Q_{t+1}\leftarrow  \delta_{x,a}^T Q_t + (1-\gamma)^{-1}(\tilde{d})^T (R + \gamma P^\pi Q_t - Q_t),
\end{align*}
we still have the contraction $\left\lVert Q_{t+1}-Q^\pi \right\rVert_\infty \leq \eta \left\lVert Q_{t}-Q^\pi \right\rVert_\infty $.

\end{proof}

\coroequivalence*
\begin{proof}
By construction, the following is true
\begin{align*}
    \mathcal{M}^w Q =  \delta_{x,a}^T Q_t + (1-\gamma)^{-1}d^T (R + \gamma P^\pi Q_t - Q_t),
\end{align*}
where $d = w\odot d_{x,a}^\mu$. Recall also by construction,
\begin{align*}
    Q_{t+1}\leftarrow  \delta_{x,a}^T Q_t + (1-\gamma)^{-1}(d^\ast)^T (R + \gamma P^\pi Q_t - Q_t).
\end{align*}
Combining the above two equations directly implies the desired result.
\end{proof}

\section{Relations between different stochastic estimators}
\label{appendix:operators}
 \begin{figure*}[h]
    \centering
    \includegraphics[keepaspectratio,width=.24\textwidth]{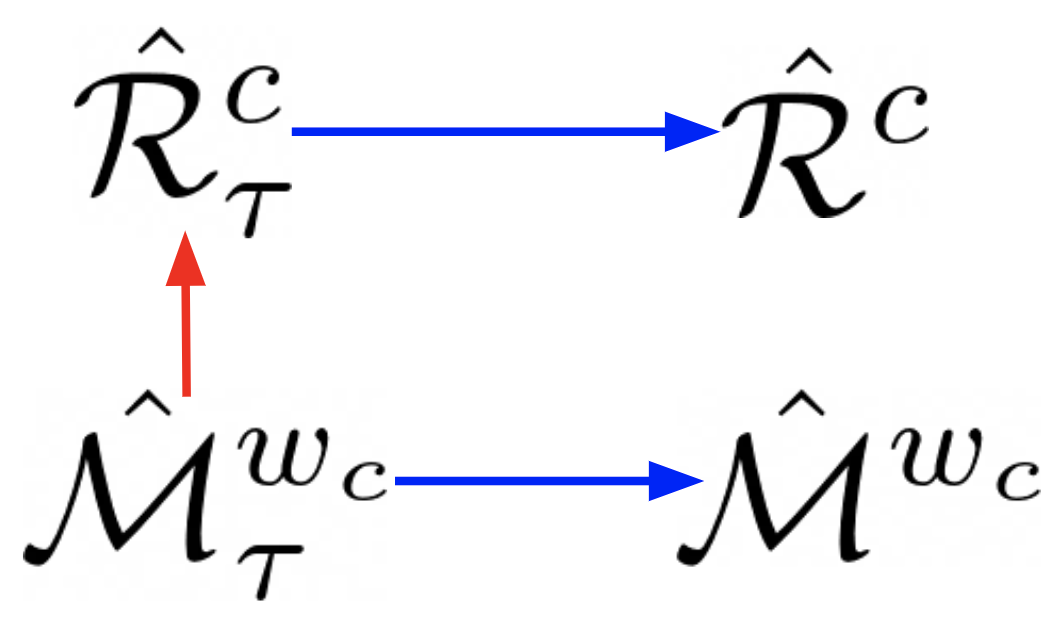}
    \caption{Visualization of relations between stochastic estimates to different operators. Blue arrows represent marginalization over the random time variables $\tau$; the red arrow represents marginalization over the random state-action pair $(x_\tau,a_\tau)$. Under suitable conditions, the directions of the arrows indicate potential variance reductions.}
    \label{fig:operators}
\end{figure*}

In Figure~\ref{fig:operators}, we show relations between stochastic estimates to different operators. Blue arrows represent marginalization over the random time variables $\tau$; the red arrow represents marginalization over the random state-action pair $(x_\tau,a_\tau)$. Under suitable conditions, the directions of the arrows indicate potential variance reductions. We see that under suitable conditions outlined in Proposition~\ref{prop:conditional-is}, we have $\mathbb{V}[\mathcal{M}_\tau^{w^c}]\geq \mathbb{V}[\mathcal{M}^{w^c}]$, $\mathbb{V}[\mathcal{R}_\tau^{c}]\geq \mathbb{V}[\mathcal{R}^{c}]$ and $\mathbb{V}[\mathcal{R}_\tau^{c}] \geq \mathbb{V}[\mathcal{M}_\tau^{w^c}]$. However, it is not clear what is the ordering of the variance between $\mathcal{R}^c$ and $\mathcal{M}^{w^c}$. 

\section{Marginalized V-trace operator}
\label{appendix:vtrace}
The V-trace operator \citep{espeholt2018impala} is defined for value functions $V\in\mathbb{R}^{|\mathcal{X}|}$. Given a target policy $\pi$ and a behavior policy $\mu$, the operator $\mathcal{R}^{c,\rho}$ is parameterized by step-wise trace coefficients $c(x,a)$ and $\rho(x,a)$. In particular,
\begin{align}
    \mathcal{R}^{c,\rho}V(x)\coloneqq V(x)+ \mathbb{E}_\mu\left\lbrack \sum_{t\geq 0} \gamma^t (c_0...c_{t-1})\rho_t \Delta_t \; \middle| \; x_0=x \right\rbrack,
    \label{eq:vtrace-multistep}
\end{align}
where $\Delta_t = \bar{r}_t+\gamma \mathbb{E}_{x'\sim p(\cdot|x_t,a),a\sim\mu(\cdot|x_t)}\left[V(x_{t'})\right] - V(x_t)$ is the TD-error at step $t$. Here, $\rho_t$ determines the fixed point of the operator, while $\rho_t,c_t$ jointly determine the contraction rate. Consider defining a marginalized V-trace operator $\mathcal{M}^{w,\rho}V(x_0)$ as below
\begin{align}
    \mathcal{M}^{w,\rho}V(x)\coloneqq V(x) + (1-\gamma)^{-1}\mathbb{E}_{x'\sim d_{x,a}^\mu}\left\lbrack  w_x(x') \rho(x',a') \Delta(x',a') \right\rbrack,
    \label{eq:vtrace-marginalized}
\end{align}
where $d_{x}^\mu(x)\coloneqq (1-\gamma)\sum_{t\geq 0} P_\mu(x_t=x'|x_0=x)$ is the discounted visitation distribution under $\mu$. here, $w(x')$ is state-dependent, $\rho(x',a')$ is state-action dependent and $\Delta(x',a')\coloneqq \bar{r}(x',a')+\gamma \mathbb{E}_{x''\sim p(\cdot|x',a'),a'\sim\mu(\cdot|x')}\left[V(x'')\right]-V(x')$.
\subsection{State-marginalized V-trace operator.} Consider setting $\rho(x,a)$ and the V-trace step-wise traces. Define TD weights $w_x^c(x')$, which is computed as 
\begin{align}
    w_x^c(x') \coloneqq \mathbb{E}_\mu\left\lbrack \sum_{t\geq0} \gamma^t (c_0...c_{t-1})\mathbb{I}[x_t=x'] \; \middle| \; x_0=x \right\rbrack.\label{eq:state-marginalized-traces-vtrace}
\end{align}
It is then straightforward to show that the multi-step operator and the marginalized operator are equivalent in expectation
$\mathcal{R}^{c,\rho}\equiv \mathcal{M}^{w,\rho}$. The trace coefficient is obtained via conditional expectation as follows.
\begin{restatable}{proposition}{propcondvtrace}\label{prop:conditional-is-vtrace}
Let $\tau$ be an integer-valued random time, such that $P(\tau=n )=(1-\gamma)\gamma^n,\forall n \geq 0$. For V-trace, given any step-wise trace coefficient $c_t$, its equivalent TD weights $w(x')$ is
\begin{align*}
    w_x^c(x') = \mathbb{E}_{\mu,\tau}\left\lbrack( \Pi_{1\leq s\leq \tau-1} c_s) \; \middle| \; x_\tau=x',x_0=x \right\rbrack.
\end{align*}
\end{restatable}
Now define $d_x^{w^c}(x')\coloneqq w_x^c(x') d_{x}^\mu(x')$. It can be shown that $d_x^{w^c}(x')\in\mathbb{R}^{|\mathcal{X}|}$ also satisfy fixed point equations.
\begin{restatable}{proposition}{propbellmanvtrace}\label{prop:trace-bellman-vtrace}
The following Bellman equation holds for the step-wise trace coefficient $c_t$ and   $d_{x,a}^{w^c}(x')$
\begin{align*}
    d_{x}^{w^c}(x^\prime)=(1-\gamma)\delta(x^\prime=x)\nonumber + \gamma \sum_{x',a'} d_{x}^{w^c}(x')c(x',a')\mu(a'|x')p(x''|x',a'),
\end{align*}
where $\delta$ is the Dirac function. Let $P^{\tilde{\pi}}\in\mathbb{R}^{|\mathcal{X}|\times|\mathcal{X}|}$ be a transition matrix such that $P^{\tilde{\pi}}(x,a)=\sum_{a'} p(y|x',a')\tilde{\pi}(a'|x')=\sum_a p(y|x',a')\mu(a'|x')c(x',a')$. Then equivalently, in matrix form,
\begin{align}
    d_{x}^{w^c}=(1-\gamma)\delta_{x} + \gamma (P^{\tilde{\pi}})^T d_{x}^{w^c}.
    \label{eq:bellman-vtrace}
\end{align}
\end{restatable}
Based on the Bellman equation in Eqn~\eqref{eq:bellman-vtrace}, it is possible to estimate $w^c(x)$ from the behavior data under $\mu$. In particular, given a starting state pair $x$, let $w_\psi(x')\approx d_{x}^{w^c}(x') \in \mathbb{R}^{|\mathcal{X}|}$ be a parameteric function used for estimating $d_{x}^{w^c}(x')$. With a critic function $\mathbf{q}\in\mathcal{Q}\subset\mathbb{R}^{\mathcal{X}}$, formulate the objective
\begin{align}
    L(\mathbf{q},w_\psi)&\coloneqq \mathbf{q}^T [(1-\gamma)\delta_x + \gamma (P^{\tilde{\pi}})^T w_\psi - w_\psi] \nonumber \\
    &= (1-\gamma)q(x) +
    \mathbb{E}_{(y)\sim d_{a}^\mu}\left\lbrack \Delta(y) \right\rbrack.
        \label{eq:marginalized-vtrace-objective}
\end{align}
Here, the TD-error $\Delta(y)\coloneqq \left(\mathbb{E}_{b\sim\pi(\cdot|y),y^\prime \sim p(\cdot|y,b)}\left\lbrack q(y^\prime)c(y,b)  \right\rbrack-q(y)\right)w_\psi(y)$. By solving a saddle point optimization problem of the above objective $\min_{w_\psi}\max_{\mathbf{q}} L(\mathbf{q},w_\psi)$, we find $w_\psi\approx d_{x}^{w^c}$.

\section{Multi-step RL algorithms}
\label{appendix:multistep}

Motivated by previous theoretical insights, we seek a practical algorithm which could combine the benefits of multi-step TD-learning and estimations of TD weights. 

\subsection{Multi-step RL algorithms with TD weights} 

We focus on the actor-critic setup where the algorithm maintains a target policy $\pi$ and a Q-function critic $Q_\theta(x,a)$. Here, the policy $\pi$ could be either parameterized $\pi=\pi_\phi$ or defined by the Q-function, e.g. the greedy policy. The algorithm collects data with behavior policy $\mu$.

To estimate TD weights, we parameterize the scoring function $q_\eta(x,a;x_0,a_0)$ and estimator $w_\psi(x,a;x_0,a_0)$, both taking as inputs the starting state-action pair $(x_0,a_0)$ and the target pair $(x,a)$. For simplicity of notations, we omit the dependency on $(x_0,a_0)$. Given a trajectory $(x_t,a_t,r_t)_{t=0}^\infty\sim\mu$, we approximate the loss function in Eqn~\eqref{eq:td-estimate} with stochastic samples,
\begin{align}
    \hat{L}(\eta,\psi)\coloneqq (1-\gamma) q_\eta(x_0,a_0) + (1-\gamma)\sum_{t=0}^\infty \gamma^t \hat{\Delta}_t,\label{eq:empirical-td-estimate}
\end{align}
where $\hat{\Delta}_t=\gamma \mathbb{E}_{\pi}\left\lbrack q_\eta(x_{t+1},\cdot)w_\psi(x_{t+1},\cdot)\right\rbrack - q_\eta(x_t,a_t)w_\psi(x_t,a_t)$. Following prior methods on scaling saddle-point optimization to neural networks (e.g. \citep{nachum2019dualdice}), we approximate the optimal solution to Eqn~\eqref{eq:td-estimate} via stochastic gradient descents (ascents) on the empirical loss $\eta\leftarrow \eta + \alpha\nabla_\eta\hat{L}(\eta,\psi),\psi\leftarrow\psi - \alpha \nabla_\psi \hat{L}(\eta,\psi)$.

At policy evaluation stage, we construct the Q-function targets with $\mathcal{M}^{w}Q(x_0,a_0)$. From the trajectory $(x_t,a_t,r_t)_{t=0}^\infty$, compute the following Q-function target
\begin{align}
    \mathcal{M}^{w}Q(x_0,a_0) \coloneqq Q_{\theta^-}(x_0,a_0) + (1-\gamma)^{-1} \sum_{t=0}^\infty \gamma^t w_\psi(x_t,a_t)\hat{\Delta}_t, \label{eq:empirical-marginalized-multistep}
\end{align}
where $\hat{\Delta}_t = r_t + \gamma \mathbb{E}_\pi\left\lbrack Q_{\theta^-}(x_{t+1},\cdot) \right\rbrack - Q_{\theta-}(x_t,a_t)$. Then the Q-function is optimized by minimizing $(Q_\theta(x_0,a_0)-\mathcal{M}^{w}(x_0,a_0))^2$.

\begin{algorithm}[h]
\begin{algorithmic}
\REQUIRE  policy $\pi$, Q-function critic $Q_\theta(x,a)$, density estimator $w_\psi(x)$, critic $q_\eta$ and learning rate $\alpha \geq 0$ \\
\WHILE{not converged}
\STATE 1. Collect data $(x_t,a_t,r_t)_{t=0}^\infty\sim\mu$ and save to the buffer $\mathcal{D}$
\STATE 2. Construct the empirical loss for marginalized estimation based on Eqn~\eqref{eq:empirical-td-estimate}. Optimize $\eta,\psi$ by alternating gradient descents (ascents): $\eta\leftarrow \eta + \alpha\nabla_\eta\hat{L}(\eta,\psi),\psi\leftarrow\psi - \alpha \nabla_\psi \hat{L}(\eta,\psi)$.
\STATE 3. Construct Q-function targets based on Eqn~\eqref{eq:empirical-marginalized-multistep}. Optimize $\theta$: $\theta\leftarrow\theta-\alpha \nabla_\theta (Q_\theta-\mathcal{M}^{w}Q(x_0,a_0))^2$.
\STATE 4. Improve the policy by either policy gradient or being greedy with respect to the new Q-function $Q_\theta(x,a)$. 
\STATE 5. Update target network $\theta^-\leftarrow\theta$.
\ENDWHILE
\caption{Multi-step RL with TD weights}
\end{algorithmic}
\end{algorithm}

\section{Fenchel-duality based approach to estimating TD weights}
\label{appendix:fenchel-dual}

In this section, we introduce Fenchel-duality based approaches to off-policy evaluation \citep{nachum2019dualdice,zhang2020gendice,nachum2020reinforcement}. While different in details, a common feature of this family of work is to convert the off-policy evaluation problem into a convex-concave optimization problem. Here, we focus on the initial formulation \emph{Dualdice}. We start by introducing this algorithm and then discuss how to extend this framework to estimate TD weights.

\subsection{Background on Dualdice}
Following \citep{nachum2019dualdice}, consider the following optimization problem with argument $w\in\mathbb{R}^{\mathcal{X}\times\mathcal{A}}$
\begin{align*}
    J(w) = \frac{1}{2}\mathbb{E}_{(x',a')\sim d_{x,a}^\mu}\left\lbrack w^2(x',a') \right\rbrack - \mathbb{E}_{(x',a')\sim d_{x,a}^\pi}\left\lbrack w(x',a') \right\rbrack.
\end{align*}
This objective is minimized at $w(x',a') = \frac{d_{x,a}^\mu(x',a')}{d_{x,a}^\pi(x',a')}$. Note that the original derivation from \citep{nachum2019dualdice} focuses on the discounted visitation distribution $d_{x}^\mu(x'))$ without conditioning on the initial action $a_0$ as in our case $d_{x,a}^\mu(x',a')$, because they focus on policy evaluation of a single starting state $x$. However, it is straightforward to extend their results. By Fenchel duality, one could further show that the above optimization could be transformed into the following saddle-point problem with $v,\psi\in\mathbb{R}^{\mathcal{X}\times\mathcal{A}}$,
\begin{align}
    \min_v \max_\psi \  \mathbb{E}_{x',a'\sim d_{x,a}^\mu,x''\sim p(\cdot|x,a),a''\sim \pi(\cdot|x')} \left\lbrack (v(x',a') - \gamma v(x'',a''))\psi(x',a') - \frac{\psi^2(x',a')}{2} \right\rbrack - (1-\gamma) v(x,a).
    \label{eq:convex}
\end{align}
The main motivation for proposing  the saddle-point optimization problem is to bypass the double sampling issue \citep{baird1995residual}. The saddle point of Eqn~\eqref{eq:convex} is $(v^\ast,\psi^\ast)$ and $\psi^\ast(x',a')=\frac{d_{x,a}^\pi(x',a')}{d_{x,a}^\mu(x',a')}$. See \citep{nachum2019dualdice} for details

\subsection{Fenchel duality-based estimation for TD weights}
Now we introduce the extension to TD weights. Given a step-wise trace coefficient $c(x,a)$ and its equivalent TD weights $w^c(x,a)$, recall that we define $d_{x,a}^{w^c}(x',a')\coloneqq w^c(x',a')\cdot d_{x,a}^\mu(x',a')$. Consider the following objective, whose optimal solution is $w_{x,a}^c(x',a')$. 
\begin{align}
    \arg\min_w J(w) = \frac{1}{2}\mathbb{E}_{(x',a')\sim d_{x,a}^\mu}\left\lbrack w^2(x',a') \right\rbrack - \mathbb{E}_{(x',a')\sim d_{x,a}^{w^c}}\left\lbrack w(x',a') \right\rbrack.
    \label{eq:init-obj}
\end{align}
First, we define $\tilde{\pi}(a|x)\coloneqq c(x,a)\mu(a|x)$. We assume that $c(x,a)$ is such that $\tilde{\pi}(\cdot|x)$ is a sub-probability vector. this is satisfied in the context of general Retrace ($c(x,a)\leq \frac{\pi(a|x)}{\mu(a|x)}$ \citep{munos2016safe}). 

Now, define variables $v(x',a')$ such that $v(x',a') = w(x',a') + \gamma \mathbb{E}_{x''\sim p(\cdot|x',a'),a''\sim \tilde{\pi}(\cdot|x'')}\left\lbrack v(x'',a'') \right\rbrack$. Note that such a quantity $v(x',a')$ exists and is unique. To see why, it is straightforward to verify that $v(x',a')$ is the fixed point of the operator $\mathcal{T}^{\tilde{\pi}}$, defined as  $\mathcal{T}^{\tilde{\pi}}Q(x',a')\coloneqq w(x',a') + \mathbb{E}_{x''\sim p(\cdot|x',a'),a''\sim \tilde{\pi}(\cdot|x'')}\left\lbrack Q(x'',a'') \right\rbrack$. Because $\gamma<1$ and $\tilde{\pi}(\cdot|x'')$ is a sub-probability vector, this operator is contractive and has a unique fixed point. As a result, starting from $w(x',a')$, by applying $(\mathcal{T}^{\tilde{\pi}})^k w(x',a')$ and let $k\rightarrow \infty$ we obtain $v(x',a')$. In vector notations, the second term of Eqn~\eqref{eq:init-obj} writes
\begin{align*}
    (d_{x,a}^{w^c})^T w = (d_{x,a}^{w^c})^T (\mathbf{v} - \gamma P^{\tilde{\pi}} \mathbf{v}) = (d_{x,a}^{w^c} - \gamma (P^{\tilde{\pi}})^T d_{x,a}^{w^c})^T \mathbf{v} = (1-\gamma)\delta_{x,a}^T v = (1-\gamma)v(x,a),
\end{align*}
where the second to last equality stems from the Bellman equation of $d_{x,a}^{w^c}$ in Eqn~\eqref{eq:trace-bellman}. 

The integrand of the first term can be rewritten as $\left((v-\gamma P^{\tilde{\pi}}v)(x',a')\right)^2$, but directly plugging in the transition matrix $P^{\tilde{\pi}}$ results in the double-sampling problem \citep{baird1995residual}. To bypass this, we follow the exact same procedure as \citep{nachum2019dualdice} and propose the following saddle-point optimization problem.
\begin{align}
    \min_v \max_\psi \  \mathbb{E}_{x',a'\sim d_{x,a}^\mu,x''\sim p(\cdot|x',a'),a''\sim \tilde{\pi}(\cdot|x'')} \left\lbrack (v(x',a') - \gamma v(x'',a''))\psi(x',a') - \frac{\psi^2(x',a')}{2} \right\rbrack - (1-\gamma) v(x,a).
    \label{eq:convex-traces}
\end{align}
The saddle point solution $(v^\ast,\psi^\ast)$ will be such that $\psi^\ast(x',a')=w_{x,a}^{c}(x',a')$.
Note that the only difference between Eqn~\eqref{eq:convex-traces} and Eqn~\eqref{eq:convex} is the target policy. Alternatively, one could interpret the new objective in Eqn~\eqref{eq:convex-traces} as executing the original dualdice algorithm but with the behavior policy $\tilde{\pi}$, which is in general a sub-probability policy.

\section{Experiment}
\label{appendix:experiment}

\subsection{Details on tabular estimations of TD weights} 
We adopt tabular representations for $w_\psi$ for both the chain MDP and Open World MDP. For tabular MDPs with $|\mathcal{X}|$ states and $|\mathcal{A}|$ actions, we represent $w_\psi$ as a $|\mathcal{X}||\mathcal{A}|\times |\mathcal{X}||\mathcal{A}|$ matrix. When both the critic $\mathbf{q}\in\mathcal{Q}$ and the estimates $w_\psi$ are tabular represented, there is no need for solving the saddle point optimization. In fact, one can directly derive solutions to the estimates given off-policy samples. We summarize the algorithmic procedure for estimating TD weights in Algorithm 2.

Given a trajectory $(x_t,a_t,r_t)_{t=0}^\infty$, Algorithm 2 specifies how to construct empirical estimates $\hat{w}$ and update the table $w_{x_0,a_0}$, i.e., the TD weights with initial state $(x_0,a_0)$. However, all state-action pairs along the trajectory could be seen as initial states. To get updates for all such pairs, we need to loop through initial pairs along the trajectory.

\paragraph{Remarks.} We can interpret Algorithm 2 as a direct implementation of the Monte-Carlo estimation to the TD weights as defined in Eqn~\eqref{eq:marginalized-trace}. This bears close resemblance to marginalized estimation techniques adopted in \citep{van2020expected}.

\begin{algorithm}[h]
\begin{algorithmic}
\REQUIRE  Table $w$ of size $|\mathcal{X}||\mathcal{A}|\times |\mathcal{X}||\mathcal{A}|$ initialized with zeros \\
\WHILE{not converged}
\STATE 1. Collect a trajectory $(x_t,a_t,r_t)_{t=0}^\infty\sim\mu$
\STATE 2. Construct cumulative step-wise traces along the trajectory: define $\hat{C}(x_t,a_t)\coloneqq (1-\gamma) \gamma^t \left(\Pi_{1\leq s\leq t} c(x_s,a_s)\right)$ for all $t\geq 0$.
\STATE 3. Accumulate cumulative step-wise traces per state-action: \begin{align*}
    \hat{w}(x,a)\coloneqq \frac{\sum_{t\geq 0} \hat{C}(x_t,a_t) \mathbb{I}[x_t=x',a_t=a']}{\sum_{t\geq 0}  \mathbb{I}[x_t=x,a_t=a]} ,\forall (x,a).
\end{align*}
If the denominator is zero, define the ratio to be zero.
\STATE 4. Update the estimate $\hat{w}(x,a)$ and set $w_{x_0,a_0}(x,a)\leftarrow (1-\alpha) w_{x_0,a_0}(x,a) + \alpha \hat{w}(x,a)$ for all $(x,a)$ with $\alpha=0.1$.
\ENDWHILE
\caption{Tabular estimation of  TD weights}
\end{algorithmic}
\end{algorithm}

\subsection{Additional experiment results}

\subsubsection{Chain MDP}

\paragraph{Details on Q-function estimation.} At each iteration $t$, the agent collects $N=1$ trajectory $(x_t,a_t,r_t)_{t=0}^\infty$. The agent maintains a Q-function $Q^{(t)}(x,a)$. Along the trajectory, we use an operator baseline to generate estimates $\hat{Q}(x_t,a_t)$. Then the Q-function is updated as $Q^{(t+1)}(x_t,a_t)\leftarrow (1-\alpha )Q^{(t)}(x_t,a_t) +\alpha \hat{Q}(x_t,a_t)$ with $\alpha=0.1$. The relative errors in Figure~\ref{fig:operators} are computed as $\sum_a  \frac{|Q^{(t)}(x_0,a) - Q^\pi(x_0,a) |}{Q^\pi(x_0,a)}$, i.e., an average measure of prediction error at the initial state $x_0$ (the leftmost state of the chain). Here, $Q^\pi(x,a)$ is computed analytically from the MDP.

 \begin{figure}[h]
    \centering
    \subfigure[\text{Soft} PI $\alpha=0.1$]{\includegraphics[width=.22\textwidth]{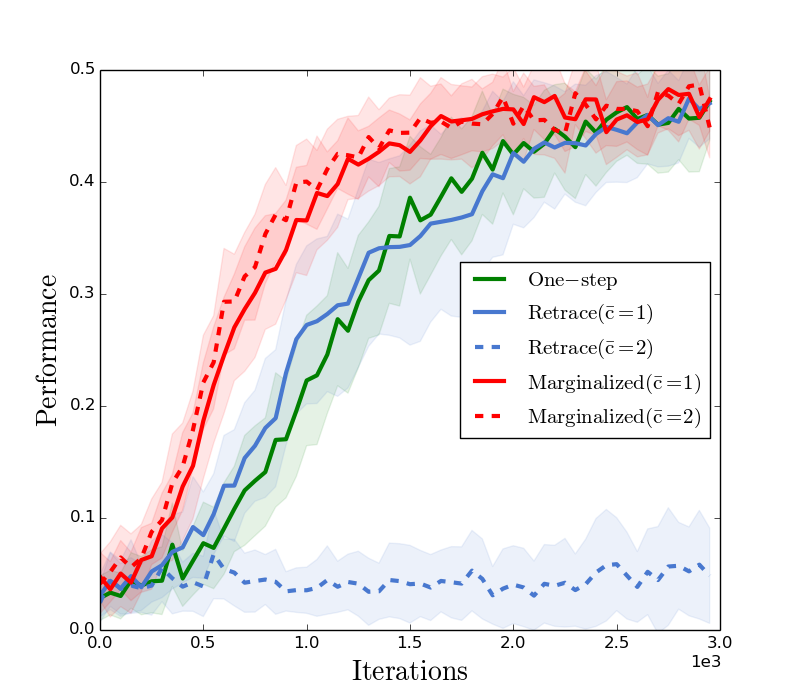}}
    \subfigure[\emph{Hard} PI $\alpha=1$]{\includegraphics[width=.22\textwidth]{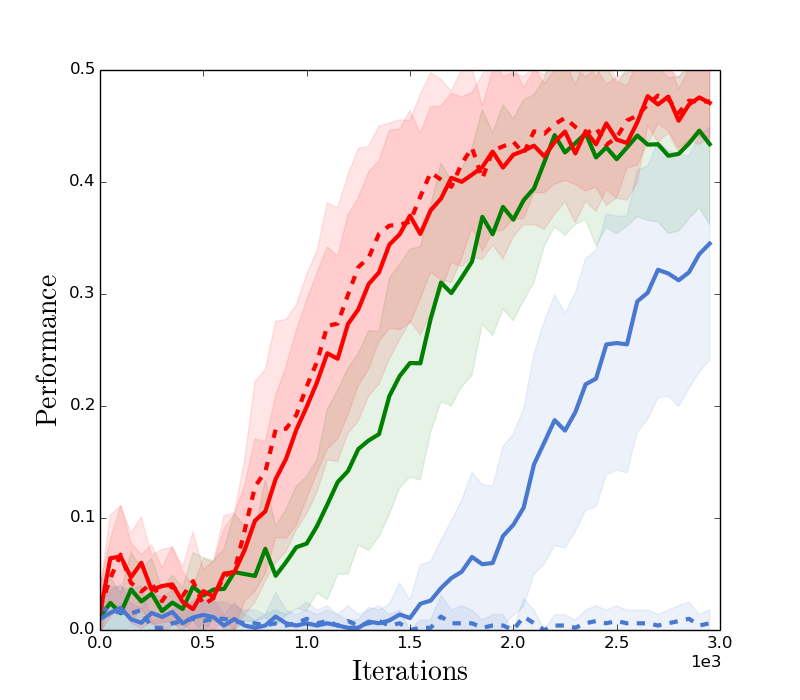}}
    \caption{Comparison of RL algorithms based on baseline operators. Each plot is averaged over $50$ runs. The x-axis shows the number of iterations and y-axis shows the performance of algorithms.}
    \label{fig:openworld-pg}
\end{figure}

\subsubsection{Open world}

\paragraph{Visualization of TD weights.} In Figure~\ref{fig:trace}, we visualize the TD weights learned by tabular representations. Recall that in general, $w_\psi\approx w^c$ is a matrix -- it takes two pairs of state-action, $(x,a)$ and $(x',a')$. Here $(x,a)$ is the initial state-action pair while $(x',a')$ is the typical argument. In the four subplots of Figure~\ref{fig:trace}, we each fix the initial location $(x,a)$ and visualize TD weights as a function of $(x',a')$ as heat maps.

Overall, we see that the learned TD weights reflect the intuition of correct credit assignment. In Figure~\ref{fig:trace}(d), where the initial state is located near the terminal state (bottom-right), it assigns low weights to most state-action pairs except near the bottom-right corner. In this case, the intuition is that Bellman errors at state-action pairs far from the bottom right should contribute much less to the estimation on average, because the random policy $\mu$ has a small chance of visiting them.

 \begin{figure*}[h]
    \centering
    \subfigure[Top left]{\includegraphics[keepaspectratio,width=.24\textwidth]{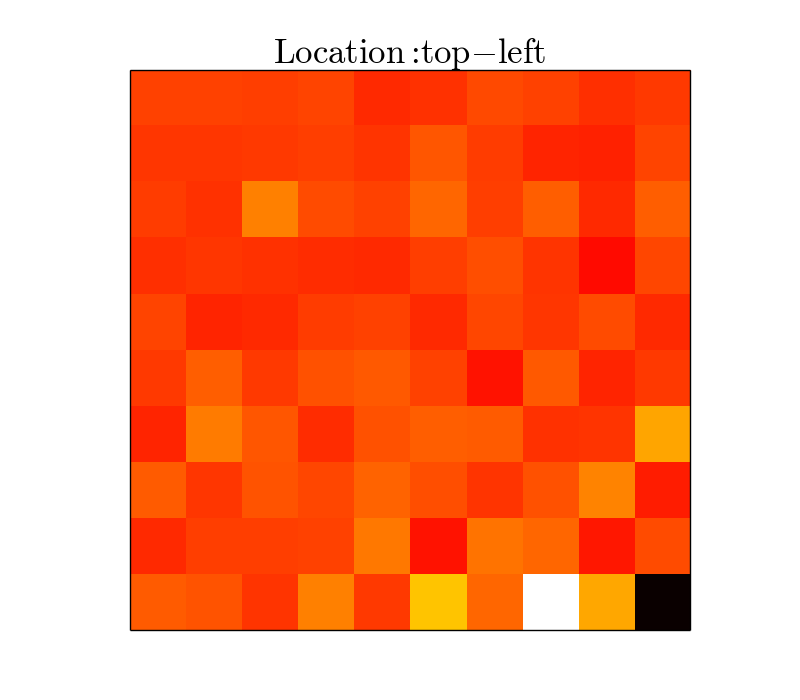}}
    \subfigure[Top right]{\includegraphics[keepaspectratio,width=.24\textwidth]{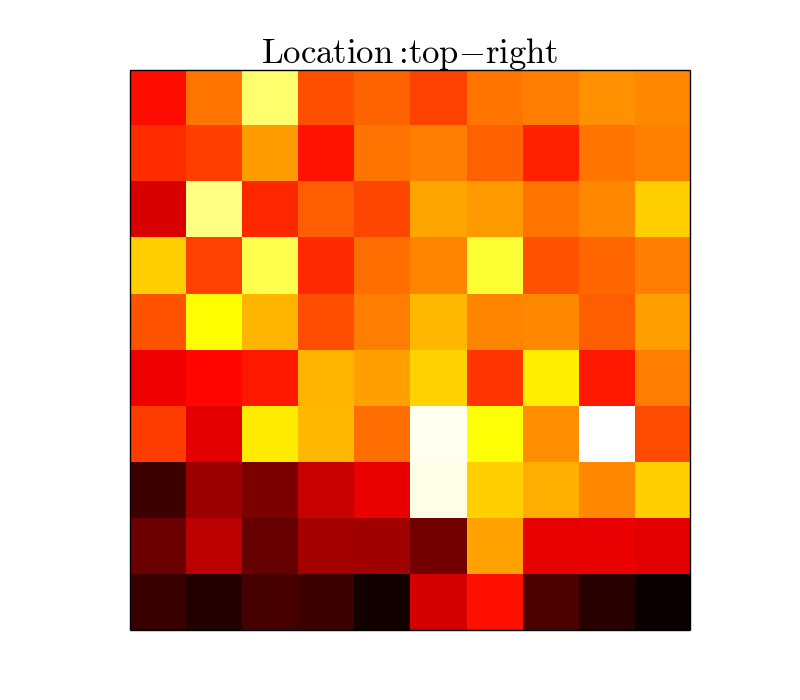}}
    \subfigure[Bottom left]{\includegraphics[keepaspectratio,width=.24\textwidth]{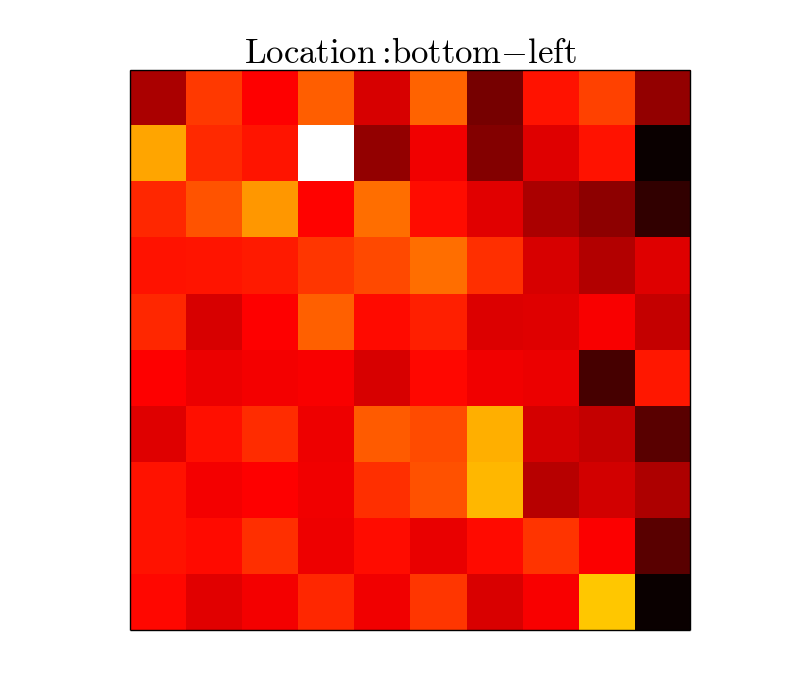}}
    \subfigure[Bottom right]{\includegraphics[keepaspectratio,width=.24\textwidth]{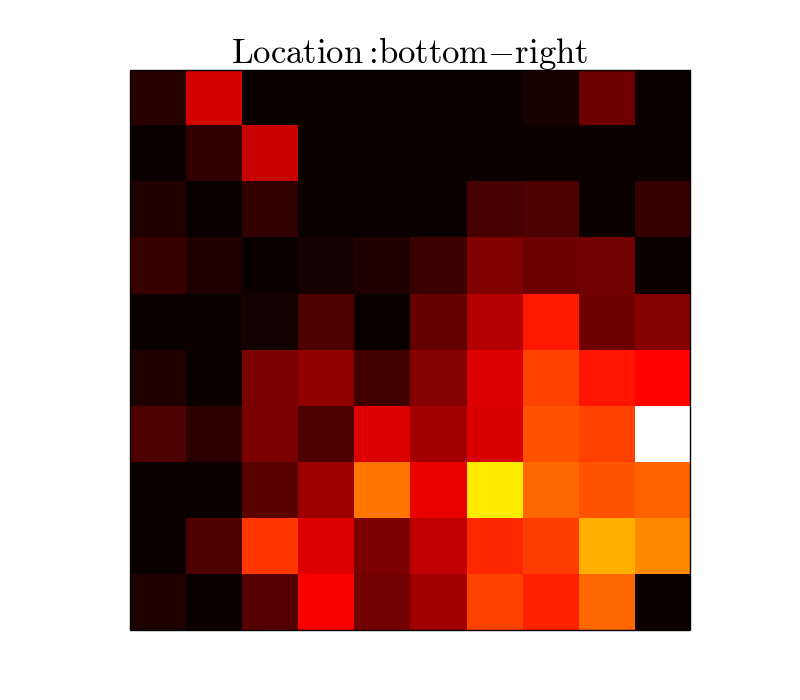}}
    \caption{Visualization of TD weights in the open world problem. All four plots show the trace estimation with different starting state, located at the top left, top right, bottom left and bottom right of the square. The trace for a state is the average of the TD weights for all actions at the state. Recall that given an initial state, the TD weights are a function of future states, which spans the entire state space.}
    \label{fig:trace}
\end{figure*}

\paragraph{Results on policy optimization.} We also consider the setup of a full off-policy optimization algorithm: policy iteration (PI): the behavior policy $\mu$ is always uniformly random, the target policy $\pi^{(0)}$ is initialized as random. At iteration $i$, the new policy is computed as $\pi^{(i+1)}=(1-\alpha)\pi^{(i)}+\alpha \pi_\text{target}$ where $\pi_\text{target}$ is the greedy policy with respect to the Q-function estimate $\hat{Q}$ at iteration $i$. In Figure~\ref{fig:openworld-pg}(a), we carry out \emph{soft} PI by setting $\alpha=0.1$; and in Figure~\ref{fig:openworld-pg}(b), hard PI by setting $\alpha=1$.

To evaluate the performance, we compare the average returns starting from uniformly random sampled states, estimated via MC estimates. For the marginalized operators, the performance gains in the \emph{one-shot} off-policy evaluation seem to carry over to the downstream optimization; however, this is not the case for Retrace. where it obtains a similar performance as the one-step operator for the soft PI; for the hard PI, because  $\pi^{(i)},i\geq 1$ are greedy policies, it is likely to cut traces quickly. In this case, Retrace does not seem to retain advantages over the one-step operator and slow down the optimization.

\subsection{Further details on deep RL experiments}

\paragraph{Benchmarks.}
For the deep RL implementations of the algorithms, we focus on continuous control tasks \citep{brockman2016openai,tassa2018deepmind}, with various simulation engines, such as MuJoCo \citep{todorov2012mujoco} and Bullet physics \citep{coumans2015bullet}. These benchmarks generally consist of locomotion tasks defined with robotics systems, with state space $\mathcal{X}$ the sensory inputs such as velocities and joints, and $\mathcal{A}$ the position or toeque controls. See documentations such as \citep{tassa2018deepmind} for details. In our experiments, we use (D) to stand for DeepMind control suite \citep{tassa2018deepmind} and (B) to stand for bullet physics \citep{coumans2015bullet}.

\paragraph{Algorithms.} We consider twin-delayed deep deterministic policy gradient (TD3) \citep{fujimoto2018addressing} as the baseline algorithm. By default, the algorithm maintains a deterministic policy $\pi_\theta(x)$ and Q-function critic $Q_\theta(x,a)$. The policy is updated by the gradients $\nabla_\theta Q_\phi(x,\pi_\theta(x))$. The critic is updated by regression against Q-function targets, such that $Q_\phi\approx Q^\pi$. Different algorithms vary in ohw the Q-function targets are defined. In general, they are defined by stochastic estimates of the evaluation operator $\mathcal{R}Q(x,a)$. For example, the vanilla TD3 constructs the target as the one-step target
$Q_\text{target}(x,a)=r(x,a) + \gamma Q_\phi(x',\pi_\theta(x'))$. TD3 also introduces a set of techniques, such as double Q-learning \citep{hasselt2010double,van2016deep} and target networks \citep{mnih2015human} to stabilize updates.

Per Algorithm 1, marginalized operators also need a density estimator $w_\psi$ and a discriminator $q_\eta$. They are trained via the objective defined in Eqn~\eqref{eq:td-estimate},
\begin{align*}
    \min_\psi \max_\eta L(q_\eta,w_\psi)=(1-\gamma)q(x,a) \nonumber +
    \mathbb{E}_{(x',a')\sim d_{x,a}^\mu,x''\sim p(\cdot|x',a')}\left\lbrack \delta(x^\prime,a^\prime,x'') \right\rbrack.,
\end{align*}
where data $(x',a')\sim d_{x,a}^\mu,x''\sim p(\cdot|x',a')$ are equivalently sampled as tuples $(x',a',x'')$ from the replay buffer.  We can construct $Q_\text{target}(x,a)=\mathcal{M}^{w_\psi}Q(x,a)$. Parameters $\psi$ and $\eta$ are optimized with alternating gradient descents (ascents). See Appendix~\ref{appendix:multistep} for further algorithmic details.

\paragraph{Baseline multi-step algorithm.} We implement a variant of Retrace \citep{munos2016safe} as the baseline multi-step algorithm. Such algorithms start with a trajectory $(x_t,r_t,a_t)_{t=0}^\infty$ starting from $(x,a)$ such that $x_0=x,a_0=a$, Q-function targets are computed recursively 
\begin{align}
    \hat{Q}_i = r_i + \gamma Q_{\phi^-}(x_{i+1},\pi_{\theta^-}(x_{i+1})) +  \gamma c_i \left( \hat{Q}_{i+1} -  \tilde{Q}_{i+1})   \right).
    \label{eq:recursive-retrace}
\end{align}
Here, parameters $\phi^-,\theta^-$ are delayed copies of the parameters $\phi,\theta$ \citep{mnih2015human}. The coefficients $c(x,a)=\lambda \cdot  \min\{1,\frac{\pi(a|x)}{\mu(a|x)}\}$. By Retrace, the Q-function $\tilde{Q}_{i+1}=Q(x_{i+1},a_{i+1})$, which we find to not work stably in practice. Instead, we use $\tilde{Q}_{i+1} = Q_{\phi^-}(x_{i+1},\pi_{\theta^-}(x_{i+1}))$. Throughout the experiments, we use $\lambda=0.7$ for the multi-step algorithms.

\paragraph{Implementation details and other hyper-parameters.} All implementations are built on SpinningUp \citep{SpinningUp2018}. Please refer to the code base for all missing details on network architecture and hyper-parameters.

\paragraph{Architecture and hyper-parameters.}
All policy networks $\pi_\theta$, Q-function networks $Q_\phi$, discirminator $q_\eta$ and estimator $w_\psi$ share the same torso networks. After the input layer, they have $2$ layers of hidden units each of size $256$. The inputs to the policy network $\pi_\theta$ are only the state variables $x$, while for all other networks are the concatenated state-action variables $[x,a]$. The discriminator output is squashed between $[-1,1]$ via $\text{tanh}(x)$ activation; the estimator $w_\psi$ output is transformed by $f(x)=\log(1+\exp(x))$ to ensure that it is strictly non-negative. Finally, the density estimator $w_\psi$ is transformed across batch $\tilde{w}(x_i,a_i) = \frac{\left(w(x_i,a_i)\right)^T}{\sum_j \left(w(x_j,a_j)\right)^T}$ to ensure stability, where $T=0.1$.

 All networks are trained with sub-sampling of mini-batches from a replay buffer. Each mini-batch is of size $100$. All networks are trained with Adam \citep{kingma2014adam} optimizers with learning rates $10^{-3}$ except for the estimator, where the learning rate is $10^{-4}$.

\end{appendix}

\end{document}